\documentclass[conference]{IEEEtran}
\IEEEoverridecommandlockouts
\usepackage[letterpaper, left=.63in, right=.63in, top=.75in, bottom=1in]{geometry}
\usepackage{cite}
\usepackage{amsmath,amssymb,amsfonts}
\usepackage{algorithmic}
\usepackage{graphicx}
\usepackage{textcomp}
\usepackage{xcolor}
\usepackage{romannum}
\usepackage{amsthm}
\usepackage{algorithmic}
\usepackage[ruled,vlined,linesnumbered]{algorithm2e}
\usepackage{float}
\usepackage{subfigure}
\SetAlFnt{\footnotesize}
\def\BibTeX{{\rm B\kern-.05em{\sc i\kern-.025emF b}\kern-.08em
T\kern-.1667em\lower.7ex\hbox{E}\kern-.125emX}}

\newtheorem{theorem}{Theorem}

\newtheorem{corollary}{Corollary}

\begin{document}

% \title{Zero-Shot Learning from Single Graphs for \\ Near-Shortest Path Routing via local search \\ in Wireless Networks}

\title{Learning from A Single Graph is All You Need for Near-Shortest Path Routing in Wireless Networks}

\author{\IEEEauthorblockN{Yung-Fu Chen, Sen Lin, and Anish Arora}
\IEEEauthorblockA{
\textit{Dept. of Computer Science and Engineering} \\
\textit{The Ohio State University}\\
Columbus, OH, USA \\
\{chen.6655, lin.4282, arora.9\}@osu.edu}
}

% \author{
% Yung-Fu Chen, Sen Lin, and Anish Arora\\
% \affaddr{\textit{Dept. of Computer Science and Engineering, The Ohio State University}, Columbus, OH, USA}\\
% \affaddr{chen.6655@osu.edu, lin.4282@osu.edu, anish@cse.ohio-state.edu}
% }

% \IEEEoverridecommandlockouts
% \IEEEpubid{\makebox[\columnwidth]{978-1-7281-7374-0/20/\$31.00~
% \copyright2020
% IEEE \hfill} \hspace{\columnsep}\makebox[\columnwidth]{ }} 
\maketitle
% \IEEEpubidadjcol

\begin{abstract}
We propose a learning algorithm for local routing policies that needs only a few data samples obtained from a single graph while generalizing to all random graphs in a standard model of wireless networks. We thus solve the all-pairs near-shortest path problem by training deep neural networks (DNNs) that efficiently and scalably learn routing policies that are local, i.e., they only consider node states and the states of neighboring nodes. Remarkably, one of these DNNs we train learns a policy that exactly matches the performance of greedy forwarding; another generally outperforms greedy forwarding. Our algorithm design exploits network domain knowledge in several ways: First, in the selection of input features and, second, in the selection of a ``seed graph'' and subsamples from its shortest paths. The leverage of domain knowledge provides theoretical explainability of why the seed graph and node subsampling suffice for learning that is efficient, scalable, and generalizable. 
%It also leads to learning one local policy that precisely matches the performance of greedy forwarding and another local policy that outperforms greedy forwarding. 
Simulation-based results on uniform random graphs with diverse sizes and densities empirically corroborate that using samples generated from a few routing paths in a modest-sized seed graph quickly learns a model that is generalizable across (almost) all random graphs in the wireless network model. 
%This paper proposes a learning algorithm for generalized local routing policies that needs only a few data samples obtained from a single graph while generalizing to all random graphs in a rich model of wireless networks. We thus solve the all-pairs near-shortest path problem by training deep neural networks (DNN) to efficiently and scalably learn routing policies that are local, i.e., they only consider node state and the states of neighboring nodes. To that end, we first design the input features for predicting action values and study the routing performance for the learning with and without exploiting domain knowledge in the input features. The study remarkably reveals that the DNNs without domain knowledge tend to learn a policy matching the performance of greedy forwarding, and the learning with domain knowledge outperforms greedy forwarding. Second, [[REWRITE THIS -- We also show from domain knowledge how to choose the seed graph and its subsamples for learning a model that generalizes to other graphs.]] we propose a heuristic [[REWRITE in terms of explanation?]] of sample selection to avoid data samples that cause overfitting. The results show that using the samples generated from a single routing path (or a few paths) in a graph quickly learn a model with comparable generalization performance. 
% In addition, a few-shot fine-tuning scheme is introduced further to improve the routing performance on a target graph. We observe a benefit with fine-tuning, especially when the routing performance is relatively low. 
\end{abstract}

\begin{IEEEkeywords}
wireless networks, local search, reinforcement learning, all-pairs shortest path routing, network knowledge
\end{IEEEkeywords}

\section{Introduction}

Routing protocol design for wireless networks is a fundamental problem that in spite of being deeply explored has room for improvement. This is largely because a one-size-fits-all solution is challenging given the scalability limits for network capacity \cite{xue2006scaling} as well as the complex network dynamics associated with (self and external) interference \cite{hekmat2004interference} and node mobility \cite{grossglauser2002mobility}.  Manually designed algorithms and heuristics often cater to particular network conditions and come with tradeoffs of complexity, scalability, and generalizability across diverse wireless networks. Human experts select appropriate routing policies for different types of networks or redesign them when the networks evolve.
%When the network changes, human experts need to redesign/refine the routing policies to adapt to target network, which prevents the current routing protocols from being applied in diverse network conditions.

%is extremely challenging due to the existence of network dynamics, such as device mobility and interference that introduce uncertainties on the network topology and connectivity. Traditional  approaches are mainly designed manually based on heuristics under particular network conditions, which inevitably suffer from the issues about complexity, scalability and generalizability. 

%To address the complexity issue, local search approaches are widely adopted in routing optimization, which can achieve reasonable performance compared to global search algorithms but with much lower time and space complexity. In routing problems, geographic routing makes routing nearly stateless by forwarding packets based on only the location information of candidate nodes and the destination. As such, the local search approaches are usually treated as the design principles for solving complicated routing problems. However,  due to the vague understanding of the achievable optimal performance with local search, most of these approaches that depend on manually designed strategies only show their feasibility but hardly convince that their performance is close to the optimum. 

Over the past decade or so, machine learning has increasingly attracted attention as an alternative data-driven albeit black-box basis for designing for wireless network routing solutions. It has for instance addressed the scalability issue for wireless network routing with large state spaces \cite{valadarsky2017learning}, with demonstrably superior performance compared to the state-of-the-art. Nevertheless, even with the use of deep neural networks (DNNs) for approximating optimal routing policies, current machine learning approaches suffer from relatively high computational complexity during training \cite{reis2019deep}. Similarly, the generalizability of the resulted routing policies over diverse network topologies and dynamics has received very little attention until recently.  Much of the evaluation of machine learned routing protocols has been conducted in small-scale and high-density networks, without demonstrating their generalizability (or scalability). Also, with increasing densification and concomitant use of resource constrained network devices, the need for learning solutions with low computational complexity has grown.

% Machine learning techniques have recently been applied to computational problems in wireless networking, which is much more challenging than the problems in wired networking due to the existence of network dynamics, such as device mobility and interference that introduce uncertainties on network topology and connectivity. In this area, routing, i.e., the network protocol design, in multi-hop wireless networks is the most fundamental problem that traditionally depends on manually designed strategies but recently attracted much attention for developing machine learning (data-driven) approaches to substitute the state-of-the-art \cite{valadarsky2017machine}.

%and also incorporate domain knowledge into the algorithm design is very critical for machine learning based approaches to be applied on resource limited devices.

% the application of machine learning in routing strategies still remains challenging. Specifically, by using DNNs to approximate routing policies, the consideration of reducing the computational complexity of training and incorporating the architecture design with domain knowledge is crucial to develop routing protocols without many computational limitations on devices and network setting assumptions. 

In this work, we focus attention on answering the following question: 
\begin{quote}
\noindent\emph{Can we design an efficient machine learning algorithm for routing based on local search that addresses complexity, scalability and generalizability issues all at once?}    
\end{quote}
\noindent 
We answer this question in the affirmative for the all-pairs near-shortest path (APNSP) problem, which serves as the foundation of many optimization problems (e.g., latency minimization and capacity maximization), for the space of unit-disk uniform random graphs. Our key insight is that ---in contrast to pure black-box approaches--- domain knowledge can be leveraged to theoretically guide the selection of ``seed'' graph(s) and corresponding sparse training data for efficiently learning models that generalize to all graphs in the chosen space. 

To motivate our focus on local search, we recall that approaches to solve the APNSP problem can be divided into two categories: global search and local search. Global search encodes the entire network state into graph embeddings \cite{narayanan2017graph2vec} and finds optimal paths, whereas local search needs only node embeddings \cite{grover2016node2vec} to predict the next forwarder on a shortest path. The model complexity (in time and space) resulting from the latter is inherently better than the former, as is the tolerance to network perturbations.  The latter can even achieve stateless design, as is illustrated by geographic routing \cite{cadger2012survey} where packet forwarding can be based on using only the location of the forwarding node and the destination. In other words, local search can achieve scalability and dynamic adaptation in a fashion that is relatively independent of the network configuration. Of course, local search comes with the penalty of sub-optimality of the chosen path relative to the optimal (i.e., shortest) path. While these assertions are true for not only manually designed solutions but also for machine learned solutions, machine learning being rooted in optimization offers the potential for outperforming manual designs that are based on heuristics.

\vspace*{1.5mm}
\noindent
\textbf{Approach.} We model the APNSP problem as a Markov decision process (MDP) and propose a DNN-based approach to learn a single-copy routing policy that only considers the states from the node and its neighbor nodes. Towards achieving efficient learning that generalizes over a rich class of graphs, we develop a little theory based on the similarity between the local ranking of node neighbors and their corresponding path length metric. If local input features can be chosen to thus achieve high similarity for most nodes in almost all  graphs in the chosen space, the APNSP objective may be realized with high probability by training a DNN that characterizes a local metric of each neighbor as a potential forwarder; the best ranked neighbor is chosen as the single-copy forwarder. The theory guides our selection of input features as well as corresponding training data and is corroborated by empirical validation of our learned routing solutions. 

%To increase the generalizability of the learned policy, the design of input features for the DNNs is independent of network topology. Specifically, given the input features that are associated with the state of one potential forwarder, we train the DNN to accurately characterize a routing metric of the forwarder, so that each node can make a routing decision by choosing the forwarder with the best metric value. 

Our approach yields a light-weight solution to generalizable routing in wireless networks in the sense that (a) the routing policy is rapidly learned from a small dataset that can be easily collected from a single ``seed'' graph; (b) the learned single-copy routing policy can be used on all nodes of a graph, and is able to generalize across diverse network sizes and densities without additional training on the target networks; and (c) the routing decision only depends on the local network state, i.e., the information of the node and its one-hop communication neighbor nodes. 

\vspace*{1.5mm}
\noindent
\textbf{Contributions.} Building on our design of a learning algorithm for generalizable routing based on local search, this paper presents the following findings:
\begin{enumerate}

\item \emph{Generalization from single graph learning is feasible and explainable for APNSP.} 

Our theory shows that with appropriate input feature design and selection of seed graphs,  a near-optimal routing policy that is generalizable across graphs can be learned from a single graph. 
The choice of a seed graph is biased towards those graphs whose nodes have the highest similarity of local ranks with respect to corresponding path metrics.

\item  \emph{Using domain knowledge for selecting input features and training samples helps increase the training efficiency.} 

We propose a knowledge-guided selection mechanism that carefully selects a seed graph as well as its data samples  for training. If input features have the desired similarity, only a few data samples, collected from nodes on a carefully chosen path in the seed graph, are sufficient for learning that guarantees the generalization performance over almost all uniform random graphs, while significantly reducing the training complexity and sample complexity. In particular, the number of required data samples for training is only in the order of $O(kd)$, where $k$ and $d$ denote the number of chosen nodes for generating samples and the average node degree, respectively. 

\item \emph{Learning from a single graph using only a distance metric matches the well-known greedy forwarding routing.} 

We first select the 
%state and action features, i.e., 
input features of the DNN based on only a distance metric. By training the DNN model with data samples collected from a single graph, we show that the performance of the learned routing policy exactly matches the performance of greedy forwarding over the uniform random graphs.

\item \emph{Learning from a single graph using not only a distance metric but also a node stretch factor relative to a given origin-destination node pair achieves even better generalized APNSP routing.} 
%One recent study \cite{chen2023qf} shows that the searching of shortest paths in uniform random graphs can be bounded within an elliptic search region with high probability. 

The domain knowledge guides the selection of a richer set of input features and corresponding choice of subsamples from a carefully selected seed graph. 
%We model this domain knowledge as a stretch factor of a node for a given origin-destination pair and incorporate it into the input feature design. Given the optimal $Q$-values that characterize the shortest path from a carefully chosen seed graph, we show that supervised learning from this single graph can 
Based on this,  a better routing policy can be learned that both achieves zero-shot generalization for APNSP and outperforms greedy forwarding over almost all uniform random graphs.

\item \emph{Reinforcement learning from a single graph  achieves comparable generalization performance for ASNSP.} 

The efficiency, scalability, and generalizability of our learning results is not limited to supervised learning. We illustrate this by developing a reinforcement learning (RL) scheme for routing without knowing the optimal $Q$-values.  By conducting experiments on random graphs across different sizes and densities, we show that learning from the seed graph(s) achieves generalization performance comparable to the supervised learning approach.

%\item A knowledge-guided few-shot fine-tuning scheme is introduced further to improve the routing performance of a given model on a target graph. 
\end{enumerate}

\iffalse
{\em Outline}.~The rest of this paper is organized as follows. In Section II,
In Section III
...in Section IV
In Section V
Finally, we discuss future work and conclusions in Section VI.
\fi

\section{Related Work}

\noindent
\textbf{Feature Selection for Routing}.~
A classic feature for local routing comes from greedy forwarding \cite{finn1987routing}, where the distance to the destination node (in a euclidean or hyperbolic metric space) is used to optimize forwarder selection. It has been proven that this feature achieves nearly optimal routing in diverse network configurations, including scale-free networks \cite{kleinberg2000navigation, papadopoulos2010greedy}. A stretch bound on routing paths using greedy forwarding is investigated in diverse models with or without the assumption of unit disk graphs \cite{flury2009greedy,  tan2009visibility, tan2011greedy, tan2009convex, won2014low}. 
%in unit disk graphs, is terms of the average node degree and the optimal length \cite{flury2009greedy}. 
Other features for forwarder selection include Most Forward within Radius (MFR) \cite{takagi1984optimal}, Nearest with Forwarding Progress (NFP) \cite{hou1986transmission}, the minimum angle between neighbor and destination (aka Compass Routing) \cite{kranakis1999compass}, and Random Progress Forwarding (RPF) \cite{nelson1984spatial}. 

Network domain knowledge has also been used to guide search efficiency in routing protocols. A recent study \cite{chen2023qf} shows that searching for shortest paths in uniform random graphs can be restricted to an elliptic search region with high probability. A Stretch Factor at each node is used as an input feature used by its geographic routing protocol, {\it QF-Geo}, to determine whether a node's neighbors lie in the search region or not and to forward packets within a predicted elliptic region.

\vspace*{1mm}
\noindent
\textbf{Generalizability of Machine Learned Routing}.~Only recently has machine learning research started to address generalizability in routing contexts. For instance, generalizability to multiple graph layout distributions, using knowledge distillation, has been studied for a capacitated vehicle routing problem \cite{DBLP:conf/nips/Bi0WCCSC22}. Some of these explorations have considered local search. For instance, wireless network routing strategies via local search based on deep reinforcement learning \cite{manfredi2021relational,  manfredi2022learning} have been shown to generalize to other networks of up to 100 nodes, in the presence of diverse dynamics including node mobility, traffic pattern, congestion, and network connectivity. Deep learning has also been leveraged for selecting an edge set for a well-known heuristic, Lin-Kernighan-Helsgaun (LKH), to solve the Traveling Salesman Problem (TSP) \cite{DBLP:conf/nips/XinSCZ21}.  The learned model generalizes well for larger (albeit still modest) sized graphs and is useful for other network problems, including routing. Likewise, graph neural networks and learning for guided local search to select relevant edges have been shown to yield improved solutions to the TSP \cite{GNN-GLS-TSP}. In related work, deep reinforcement learning has been used to iteratively guide the selection of the next solution for routing problems based on neighborhood search \cite{wu2021learning}.

\section{Problem Formulation for Generalized Routing}
We assume that a wireless network is represented by a unit disk graph (UDG) $G = (V, E)$ whose nodes are uniformly randomly distributed over a 2-dimensional Euclidean plane. Each node $v \in V$ is associated with a unique ID and knows its global coordinates. For any nodes $v, u \in V$, edge $(v, u) \in E$ if and only if the Euclidean distance between $v$ and $u$ is not larger than the communication radius $R$. The weight function $w(v, u) \leq R$ denotes the length of edge $(v, u)$. Let $\rho$ denote the network density, where network density is defined to be the average number of nodes per $R^2$ area, and $n$ the number of nodes in $V$. It follows that all nodes in $V$ are distributed in a square whose side is of length $\sqrt{\frac{n \times R^2}{\rho}}$.

\vspace*{1mm}
Typically, in the all-pairs shortest path routing problem in $G$, the objective is to determine a routing policy $\pi(O, D, v) = u$ that finds $v$'s next forwarder $u$ for any origin and destination pair $(O, D)$, such that the total length of the corresponding routing path, $p = [v_0, ..., v_L], v_0=O, v_L=D$, is minimized:
\begin{equation*}
\begin{array}{l}
\min \ \sum_{i=0...L-1 \wedge v_i \in p}{w(v_i, v_{i+1})}.
\label{}
\end{array}
\end{equation*} 

\subsection{All-Pairs Near-Shortest Path Problem}

This paper aims to solve the all-pairs near-shortest path (APNSP) routing problem for any uniform random graph. Here, we define the APNSP problem objective as finding a near-shortest path whose length is within a user-specified factor ($\geq 1$) of the shortest path length. %A solution of APNSP implies the existence of a routing path, other than the shortest path, that also satisfies a user-specific constraint. 

Formally, let $d_e(O, D)$ denote the Euclidean distance between two endpoints $O$ and $D$, and $d_{sp}(O, D)$ denote the length of the shortest path between these endpoints. Further, let $\zeta(O, D)$ denote the path stretch of the endpoints, i.e., the ratio $\frac{d_{sp}(O, D)}{d_{e}(O, D)}$.

\vspace*{1.5mm}
\textbf{The APNSP Problem.} Given a connected graph $G = (V, E)$ and any source-destination pair $(O, D)$ where $O, D \in V $, determine a routing policy $\pi(O, D, v) \! = \! u$ that finds $v$'s next forwarder $u$ and the corresponding routing path $p(O, D)$ with path length $d_{p}(O, D)$, such that with high probability the accuracy of the found path is optimized as follows:
\begin{align}
% \begin{array}{l}
\max& ~~ Accuracy_{G, \pi} = \frac{\sum_{O, D \in V}{\eta(O, D)}}{|V|^2},
\\
s.t. & ~~
\eta(O, D) = \begin{cases} 1, \ \ i\!f \ \frac{d_{p}(O, D)}{d_{sp}(O, D)} \leq  \zeta(O, D) (1+\epsilon) \\
0, \ \ otherwise
\end{cases}
\label{APNSP_Accuracy}
% \end{array}
\end{align} 

\begin{figure}[thb!]
\vspace{-2mm}
    \centering
     \subfigure
     {
        \includegraphics[width=0.48\textwidth]{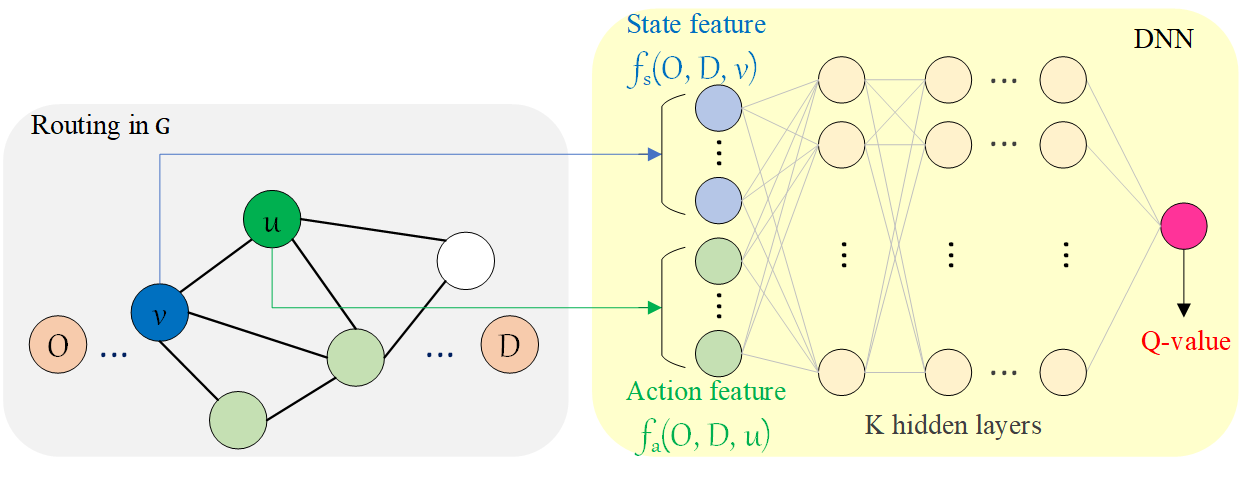}
    }
    \caption{Schema for solution using DNN to predict $Q$-values for selecting the routing forwarder.}
    \label{DNN}
    \vspace{-2mm}
\end{figure}

In other words, the user-specified factor for APNSP is $\zeta(O, D) (1+\epsilon)$, where $\epsilon\geq 0$.

\subsection{MDP Formulation for the APNSP Problem}
To solve the APNSP problem, we first formulate it as a Markov decision process (MDP) problem which learns to choose actions that maximize the expected future reward. In general, an MDP consists of a set of states $S$, a set of actions  $A(s)$ for any state $s\in S$, an instantaneous reward $r(s, a)$, indicating the immediate reward for taking action $a\in A(s)$ at state $s$, and a state transition function $P(s'|s, a)$ which characterizes the probability that the system transits to the next state $s'\in S$ after taking action $a$ at the current state $s\in S$. To simply the routing behavior in the problem, the state transition is assumed to be deterministic. Specifically, each state $s$ represents the features of a node holding a packet associated with an origin-destination pair $(O, D)$, and an action $a\in A(s)$ indicates the routing behavior to forward the packet from the node to one of its neighbors. Given the current state $s$ and an action $a\in A(s)$ which selects one neighbor as the next forwarder, the next state $s'$ is determined as the features of the selected neighbor such that the probability $P(s'|s, a)$  is always one. The tuple $(s, a, r, s')$ is observed whenever a packet is forwarded.

In addition, we define the $Q$-value to specify the cumulative future reward from state $s$ and taking action $a$: 
\vspace*{-1mm}
\begin{align*}
    Q(s_t=s, a_t=a) = \sum_{i=t}^L \gamma^{i-t} r(s_i, a_i)
\vspace*{-1mm}
\end{align*}
% The $Q$-value function is then defined as follows:
% \begin{equation*}
% \begin{array}{l}
% Q(s, a) = r + \gamma\max{Q(s', a')}  
% \label{}
% \end{array}
% \end{equation*} 
where $\gamma, 0 \! \leq \! \gamma \! \leq \! 1,$ is the discount factor. When $\gamma=0$, the instantaneous reward is considered exclusively, whereas the future rewards are treated as equally important as the instantaneous reward in the $Q$-value if  $\gamma=1$. In the APNSP problem, we define the instantaneous reward $r(s, a)$ as the negative length of the corresponding edge $(s, s')$, and set $\gamma=1$. Therefore, the optimal $Q$-value $Q^*(s, a)$ is equal to the cumulative negative length of the shortest path from $s$ to the destination.

To solve the APNSP problem, we seek to learn the optimal $Q$-value through a data-driven approach with DNNs. As depicted in Figure~\ref{DNN}, each state $s$ and action $a$ will be embedded into a set of input features denoted by $f_{s}(s)$ and $f_{a}(a)$, respectively. A DNN will be learned to approximate the optimal $Q$-value given the input features, based on which a near-shortest path routing policy can be obtained by taking actions with largest $Q$-values.  

\section{Explainable Generalizability of \\ Routing Policies}
In this section, we present a sufficient condition for how a routing policy, $\pi$, learned from a seed graph $G^{*}=(V^{*}, E^{*}) \in \mathbb{G}$, where $\mathbb{G}$ is the set of all uniform random graphs, with select samples generated from a subset of nodes in $V^{*}$, can generalize over other (and potentially all) uniform random graphs $G \! \in \! \mathbb{G}$.

A basis for generalizability in this setting is the concept of a ranking metric for each node $v \! \in \! V$ with respect to each $u \! \in \! nbr(v)$, where $nbr(v)$ denotes the set of $v$'s one-hop neighbors. Let $f_{s} \! : \! V \! \rightarrow \! \mathbb{R}^{I}$ be a map of $v \in V$ to its $I$ state features and let $f_{a} \! : \! V \! \rightarrow \! \mathbb{R}^{J}$ be a map of $u \! \in \! nbr(v)$ to its $J$ action features. We define a ranking metric $m(f_{s}(v),f_{a}(u)) \! \in \! \mathbb{R}$ to be a {\em linear} function over the input features associated with $v$ and $u$. For notational convenience, given an ordering $\langle u_0, ..., u_d \rangle$ of all nodes in $nbr(v)$, let $X_{v} = \{(f_{s}(v),f_{a}(u_0)), ..., (f_{s}(v),f_{a}(u_d))\}$, denote the set of input vectors for each corresponding neighbor $u_k \! \in \! nbr(v), 0 \! \leq \! k \! \leq d$. Also, let $Y_{v} = \{Q(v, u_0), ..., Q(v,u_d)\}$ denote the corresponding set of $Q$-values. 
%For notational convenience, let $X_{v} = \{m(f_{s}(v),f_{a}(u_0)), ..., m(f_{s}(v),f_{a}(u_d))\}, u_0, ..., u_d \in nbr(v)$ denote the set of measures for all of $v$'s neighbors. Also, let $Y_{v} = \{Q(v, u_0), ..., Q(v, u_d)\}$ be the corresponding set of $Q$-values. 

%%%YUNG-FU:  Let's discuss the next statement since it is not clear. Also note that the reader has not seen the figure for the solution yet, so some terms like action features are not clear.

%For each state represented by a node $v$ forwarding a packet, the set of routing actions is collected from all its neighbors $u \in nbr(v)$, where $nbr(v)$ denotes the set of $v$'s one-hop neighbors. 

%of action features at state $v$, and $X_{v} = \{m(u_0), ..., m(u_d)\}, u_0, ..., u_d \in nbr(v)$ denote the set of measures collected from all $v$'s neighbors. 

%We let $m(u)$ denote the measure\footnote{The measure, $m(u) \in \mathbb{R}$, is defined to be the output of a linear or non-linear function with input of all $c$ action features $[a_0(u), ... a_c(u)]$.} of action features at state $v$, and 

We begin with a sufficient condition on the relation between the ranking metric $m$ and the corresponding $Q$-values set $Y_{v}$ to learn a DNN model for ranking the neighbors $u \in nbr(v)$ according to their $Q$-values. 

%Let $RK^{+}$ and $RK^{-}$ be the ranking function sorting elements in a list with an ascending and descending order. Then $RK^{+}(X_{v})$ and $RK^{-}(Y_{v})$ is the ranking sequence of $v$'s measure list and $Q$-value list from $u$. 

\begin{corollary}
Let $v$ be any node in $V$ for which ranking metric $m(f_{s}(v), f_{a}(u))$ satisfies the following property, {\bf RankPres}: 
\begin {quote}
If $\langle m(f_{s}(v), f_{a}(u_0)), \, ... \, , \, m(f_{s}(v), f_{a}(u_d ) \rangle$ is monotonically increasing then $\langle Q(v, u_0), ..., Q(v, u_d)\rangle$ is monotonically increasing. \end{quote}
There exists a learnable DNN $H$, $H: \! \mathbb{R}^{I+J} \! \rightarrow \! \mathbb{R}$, with training samples $\langle X_{v}, Y_{v} \rangle$, that achieves optimal ranking of all $u \in nbr(v)$, i.e., its output for the corresponding neighbors of $v$, $\langle\, H(f_{s}(v), f_{a}(u_{0})), ..., H(f_{s}(v), f_{a}(u_{d})) \, \rangle$, is monotonically increasing.
\label{Learnability}
\end{corollary}    
\begin{proof}
Since the ranking metric $m$ is a linear function, a single neuron DNN $H$ can achieve the optimal ranking by weighing each input feature with the weight corresponding to that feature in $m$.
%Let $m()$ be a linear function and the corresponding coefficient for the $i$ state features ($x_{0}, ..., x_{i-1}$) and $j$ action features ($x_{i}, ..., x_{i+j-1}$) are $c_{0}, c_{1}, ... c_{i+j-1}$. Thus, $m(f_{s}(v),f_{a}(u)) = c_{0}x_{0} + ...+ c_{i+j-1}x_{i+j-1}$
%We can simply construct the DNN $H$ with one hidden layer and a single neuron that has weight, $c_{0}, c_{1}, ... c_{i+j-1}$, respectively assigned to the input feature $x_{0}, ..., x_{i+j-1}$. The function $H(f_{s}(v), f_{a}(u)) = m(f_{s}(v), f_{a}(u))$ is thus learnable as a linear function whose output preserves the monotonicity for $X_{v}$.
\end{proof}

%We note that the weight of $v$ does not affect the routing decision at node $v$ because $v$ is invariant to all its corresponding $Q$-values $Q(v, u_0), ..., Q(v, u_d)$.

Next, we lift the sufficient condition to provide a general basis, first, for ranking the neighbors of all nodes in a graph according to their optimal shortest paths, from only the samples derived from one (or a few) of its nodes; and second, for similarly ranking the neighbors of all graphs in $U$.

\begin{theorem}
[Cross-Node Generalizability] 
%For any graph $G$, $G \! = \! (V,E)$, if for all $v \! \in \! V$ there exists a  sequence of $u \! \in \! nbr(v)$, $\langle u_0, ..., u_d \rangle$, and a measure function $m(f_{s}(v), f_{a}(u))$ such that $[m(f_{s}(v), f_{a}(u_0)), ..., m(f_{s}(v), f_{a}(u_d))]$ and the corresponding optimal $Q$-value ordering, $[Q^{*}(v, u_0), ..., Q^{*}(v, u_d)]$, are respectively monotonically increasing, then an optimal routing policy for all $v \! \in \! V$ is learnable with only the subset of training samples $\langle X, Y \rangle$, where $X \! = \! \bigcup_{v \in V' \subseteq V \wedge u \in nbr(v)} \ {X_{v}}$ and $Y \! = \!  \bigcup_{v \in V' \subseteq V \wedge u \in nbr(v)} {[Q^{*}(v, u_0), ..., Q^{*}(v, u_d)]}$.
For any graph $G$, $G \! = \! (V,E)$, if there exists a ranking metric $m(f_{s}(v), f_{a}(u))$ that satisfies the RankPres property for all $v \! \in \! V$, then an optimal ranking policy for all $v \! \in \! V$ is learnable with only a subset of training samples $\langle X_{V'}, Y_{V'} \rangle$, where $V' \subseteq V$, $X_{V'} \! = \! \bigcup_{v \in V' }{X_{v}}$,
and $Y_{V'} \! = \!  \bigcup_{v \in V'}{Y_{v}}$.
\label{OPT_Routing}
\end{theorem} 

%\wedge u \in nbr(v)} \ {X_{v}}$ 
%{\langle Q^{*}(v, u_0), ..., Q^{*}(v, u_d)\rangle}$.
%$ \, \wedge \, V' \subseteq V}$
%\wedge u \in nbr(v)} 

\begin{proof}
From Corollary~\ref{Learnability}, $m$ can be learned by a DNN with samples generated from any $v \in V$ (or, more generally, from any set of nodes $V' \subseteq V$). Since the RankPres property holds for all $v \in V$, the resulting DNN achieves an optimal ranking of all nodes in $V$ and hence for $G$.
\end{proof}

%Since $m$ is linear, based on corollary 1, we can learn $m$ with only a subset of training samples. After we find a DNN that learns $m$, based on the RankPres property, we can find the optimal routing policy for all nodes.

%Since the monotonicity of the optimal $Q$-values and the measure function $m()$ is preserved for all $v \in V$ associated with a particular order $u_0, ..., u_d$ for given $v$, a set of samples generated from any $v \in V$ is possible to learn a linear function that preserves the monotonicity: $m(f_{s}(v), f_{a}(u_1)) < m(f_{s}(v), f_{a}(u_2)) \Rightarrow Q^{*}(v, u_1) > Q^{*}(v, u_2)$.
%CLEANUP:  Why is Q* used and not Q?  Why does Q* imply optimal routing?

Note that if  the $Q (v,u)$ value corresponds to the optimal (shortest) path $Q$-value for each $(v, u)$ pair, then the DNN indicated by Theorem~\ref{OPT_Routing} achieves an optimal routing policy for all nodes in $V$. Note also in this case that if the ranking metric $m$ satisfies {\em RankPres} not for all nodes but for almost all nodes, a policy learned from samples from one or more nodes $v$ that satisfy {\em RankPres} needs not to achieve optimal routing for all nodes. Nevertheless, if the relative measure of the number of nodes that do not satisfy {\em RankPres} to the number of nodes that do satisfy {\em RankPres} is small, it is still empirically possible that with high probability the policy achieves near-optimal routing, as we will later experimentally validate.

%if the monotonicity of optimal $Q$-values with respect to the ranking metric $m$ is preserved across almost all nodes, $v \in V_{M} \subseteq V \wedge |V_{M}| \approx |V|$, then a near-optimal routing policy can be learned by using training samples collected from a subset of nodes from $V_{M}$. Therefore, a near-optimal routing problem can be solved by finding a measure function of action features, $m: A^{k} \rightarrow \mathbb{R}$, that suffices the property of monotonicity with respect to the $Q^{*}$-values.

\begin{theorem}
[Cross-Graph Generalizability]
If there exists a ranking metric $m(f_{s}(v), f_{a}(u))$ that satisfies the RankPres property for the nodes in all graphs $G \! \in \! U$, then an optimal ranking policy is learnable by using training samples from one or more nodes in one (or more) chosen seed graph(s) $G^{*} \! \in \! U$.
\label{routing_generalizability}
\end{theorem} 
%$\forall^{\infty} G \in U$, $SIM_{G}(m, Q^{*}) \approx 1$ and $\exists G^{*}$ s.t. $SIM_{G^{*}}(m, Q^{*}) \approx 1$  $\implies$ a near-optimal routing policy is learnable that achieve generalizability across all graph $\in U$, by using training samples from $G^{*}$

%In addition, there is a rich set of $G \in U$ preserving the monotonicity of the optimal $Q$-values and the measure function $m()$. $\pi$ achieves the near-optimal routing performance on $G$. Let $U'$ be the set of graph $G$ such that $SIM_{G}(m, Q^{*}) \approx 1$. Because $\frac{|U'|}{|U|} \approx 1$ and $\frac{\sum_{G \in U'}{SIM_{G}(m, Q^{*})}}{|U|} \approx 1$, the routing policy $\pi$ achieve generalizability across all uniform random graphs in $U$.

\begin{proof}
From Theorem~\ref{OPT_Routing}, learning from $G^{*}$ with training samples chosen from one or more nodes in $G^{*}$ achieves an optimal ranking policy $\pi$ for all nodes in $G^{*}$. The soundness of the subsampling argument suffices for the same policy to achieve optimal ranking for nodes in any graph $G \in U$, as long as there exists a common ranking metric $m$ with the RankPres property holding for all graphs in $U$.
\end{proof}

Again, if Theorem~\ref{routing_generalizability} is considered in the context of $Q$-values corresponding to optimal shortest paths, the learned routing policy $\pi$ generalizes to achieving optimal routing over all graphs $G \! \in \! U$. And if we relax the requirement that {\em RankPres} holds for all nodes of all graphs in $U$ to only requiring that for almost all graphs $G \! \in \! U$, there is a high similarity between the ranking metric $m$ and the optimal $Q$-value, it becomes empirically possible that, with high probability, the policy achieves near-optimal routing. To this end, the choice of the seed graph $G^{*}$ and its subsampled nodes is biased so that the similarity is maximized. (In the experimental evaluation presented in the next section, we adopt Discounted Cumulative Gain (DCG) \cite{jarvelin2002cumulated} for formalizing the notion of similarity.)

To summarize, an efficient solution to a generalizable near-optimal routing problem is made feasible by choosing input features, for which some ranking metric preserves monotonicity with respect to the optimal route $Q$-values with high probability across the set of all graphs $U$. For the APNSP routing problem, the optimal $Q(v, u)$ values are retrieved by calculating the length of the shortest path starting from $v$ toward $u$ until reaching the destination. A near-optimal routing policy may then be learned via supervised learning on a single seed graph.
%if the property of monotonicity of $m()$ and $Q^{*}()$ can be shown across a rich class of graphs and a seed graph used for training sample collection is discoverable.
%More generally, to capture the notion that a ranking metric is largely preserved for the graph $G$, let us define the {\em ranking similarity} between the ranking metric and the optimal $Q$-values at node $v$, $SIM_{v}(m, Q^{*}) \in [0,1]$, as follows. If the sorted ascending order of $m(u), u \in nbr(v)$ is close to the sorted ascending order of $Q^{*}(v, u)$, the similarity is likely be 1.
%Moreover, let $SIM_{G}(m, Q^{*}) \in [0,1]$ be the average ranking similarity between $m()$ and $Q^{*}()$ across all node $v \in V$ in a given graph $G=(V,E)$, i.e., $SIM_{G}(m, Q^{*}) = \frac{\sum_{v \in V}{SIM_{v}(m, Q^{*})}}{|V|}$. Let $\forall^{\infty}$ represent the meaning ``almost all''. The following theorem demonstrate the generalizability of a routing policy to a rich class of uniform random graphs.

\section{Single Graph Learning}
% We develop both supervised and reinforcement learning approaches to train the DNNs.
Building on the above analysis, we propose a single seed graph learning method that generalizes with high probability over almost all uniform random graphs in $U$.  We guide input feature selection, in Section V.A, from network knowledge and verify the respective high ranking similarity with optimal routes, in Section V.B. %The analysis above shows that \emph{learning from a carefully chosen single graph with appropriate input feature designs will be able to generalize well over uniform random graphs.} Inspired by this, in this section we propose a single graph learning method with knowledge guided feature designs.

\subsection{Design of Input Features}

To learn a generalizable routing protocol across graphs with different scales, connectivities, and topologies, the input feature design of the DNN should be independent of global network configurations. 
%such that the learned $Q$-function is generalizable. 
%Each node uses the same features to predict one single $Q$-value of each neighbor node, and then the routing decision is made by choosing the neighbor node with the largest $Q$-value.
Therefore, we exclude any feature containing ID information specific to nodes or packets. Recall that each node knows its own coordinates and the coordinates of the origin and the destination. 
% Note that the coordinate query service can be implemented by equipping GPS on nodes and deploying an oracle for handling out-of-band queries of node location in wireless networks. 

%For a node $v$, let $nbr(v)$ denote the set of its one-hop neighbors. 
As motivated in Section II, input features based on Euclidean distance and the stretch factors have been found useful in local geographic routing protocols. Accordingly, we design the input features, including the state and action features, as follows:

\begin{itemize}
\item State feature, $f_{s}(O, D, v)$. 
For a packet with its specified origin $O$ and destination $D$ at node $v$, the state features are the vectors with the elements below.
    \begin{enumerate}
    % \item \textbf{O-D distance, $\overline{O D}$:} The Euclidean distance between origin $O$ and destination $D$. This feature is also used to distinguish the routing for different origin and destination pairs.
    \item \textbf{Distance to destination, $\overline{v D}$:} The Euclidean distance between node $v$ and the destination $D$.
    \item \textbf{Stretch factor, $S\!F_{O, D, v}$:} The stretch of the indirect distance between $O$ and $D$ that is via node $v$ with respect to the direct distance between $O$ and $D$. 
    %It captures the extent of the search region to find the shortest path. 
    \begin{equation*}
    \begin{array}{l}
    S\!F_{O, D, v} = \frac{\overline{O v}+\overline{v D}}{\overline{O D}}.
    \label{}
    \end{array}
    \end{equation*} 
    \end{enumerate}

\item Action feature, $f_{a}(O, D, a) = f_{s}(O, D, u)$. Given the origin $O$ and destination $D$, the feature for the action that forwards a packet from node $v$ to node $u \in nbr(v)$, $f_{a}(O, D, a)$, is chosen to be the same with the state feature of $u$, $f_{s}(O, D, u)$. 
\end{itemize}

In what follows, we consider learning with two different combinations of input features, one with only $\overline{v D}, \overline{u D}$ and the other with both $\overline{v D}, \overline{u D}$ and $S\!F_{O, D, v}, S\!F_{O, D, u}$.

\subsection{Ranking Similarity between $m$ for Input Features and $Q^{*}$}

Based on the little theory of Section~IV, we now investigate the existence of a ranking metric $m$ for the first input feature, as well as the pair of input features, and demonstrate that in both cases the ranking similarity is close to 1 across nodes in almost all graphs $G$, regardless of their size and density.  The results thus empirically corroborate the feasibility of routing policy generalization from single graph learning, and also guide the selection of seed graphs and training samples.

Let $SIM_{v}(m, Q^{*}) \in [0,1]$ denote the {\em ranking similarity} between the ranking metric $m$ and the optimal $Q$-values ($Q^{*}$) at node $v$.  Moreover, let $SIM_{G}(m, Q^{*}) \in [0,1]$ denote the average ranking similarity between $m$ and $Q^{*}$ across all nodes $v \in V$ in a given graph $G=(V,E)$, i.e., $SIM_{G}(m, Q^{*}) = \frac{\sum_{v \in V}{SIM_{v}(m, Q^{*})}}{|V|}$.

\vspace*{1.5mm}
Conceptually, $SIM_{v}(m, Q^{*})$ should tend to 1 as the order of neighbors in the sorted ascending order of $m(u), u \in nbr(v)$, comes closer to matching the order of the neighbors in the sorted ascending order of $Q^{*}(v, u)$. More specifically, we adopt Discounted Cumulative Gain (DCG) \cite{jarvelin2002cumulated} to formally define the ranking similarity. The idea of DCG is to evaluate the similarity between two ranking sequences by calculating the sum of graded relevance of all elements in a sequence. First, a sorted sequence of $Q^{*}(v, u)$ with length $L = |nbr(v)|$, $u \! \in \!  nbr(v)$, is constructed as the ideal ranking, denoted by $A$. For each position $i$ in $A$, we assign a graded relevance $rel_{A}[i]$ by following the rule: $rel_{A}[i] = (L - i)^2, i = 0 ... L-1$\footnote{The assignment of graded relevance could be any way to decrease the value from left to right positions. Here we use squared value to assign dominant weights to the positions close to leftmost.}. The value of DCG accumulated at a particular rank position $\tau$ is defined as:
\begin{equation*}
\begin{array}{l}
DCG_{\tau} = \sum_{i=1}^{\tau}{\frac{rel[i]}{\log_{2}(i+1)}}.
\label{DCG}
\end{array}
\end{equation*} 
For example, let $A = [4, 1, 3, 2, 5]$. The corresponding $rel_{A}[i]$ and $\frac{rel_{A}[i]}{\log_{2}(i+1)}$ values are shown in Table~\ref{ideal_DCG_table}. Then $A$'s $DCG_{3} = 25+10.095+4.5 = 39.595$.

\begin{table}[htb!]
\caption{An example of DCG calculation for an ideal ranking $A=[4, 1, 3, 2, 5]$.}
\begin{center}
\begin{tabular}{|c|c|c|c|c|}
\hline
\textbf{$i$} & \textbf{$A[i]$} & \textbf{$rel_{A}[i]$} & \textbf{$\log_{2}(i+1)$} & \textbf{$\frac{rel_{A}[i]}{\log_{2}(i+1)}$} \\
\hline
1 & 4 & 25 & 1 & 25\\
2 & 1 & 16 & 1.585 & 10.095\\
3 & 3 & 9 & 2 & 4.5\\
4 & 2 & 4 & 2.322& 1.723\\
5 & 5 & 1 & 2.807 & 0.387\\
\hline
\end{tabular}
\label{ideal_DCG_table}
\end{center}
\vspace{-2mm}
\end{table}
\vspace*{-4mm}
\begin{table}[htb!]
\caption{An example of DCG calculation for an estimated ranking $B = [1, 2, 4, 5, 6]$.}
\begin{center}
\begin{tabular}{|c|c|c|c|c|}
\hline
\textbf{$j$} & \textbf{$B[j]$} & \textbf{$rel_{B}[j]$} & \textbf{$\log_{2}(j+1)$} & \textbf{$\frac{rel_{B}[j]}{\log_{2}(j+1)}$} \\
\hline
1 & 1 & 16 & 1 & 16\\
2 & 2 & 4 & 1.585 & 2.524\\
3 & 4 & 25 & 2 & 12.5\\
4 & 5 & 1 & 2.322& 0.431\\
5 & 6 & 0 & 2.807 & 0\\
%5 & 3 & 9 & 2.807 & 3.482\\
\hline
\end{tabular}
\label{estimated_DCG_table}
\end{center}
\vspace{-2mm}
\end{table}

Next, a sorted sequence of $m(u)$ with length $L = |nbr(v)|$, $u \in nbr(v)$,  is constructed as the estimated ranking, denoted by $B$. Let $B = [1, 2, 4, 5, 6]$. The graded relevance for $B[j]$ depends on the position of $B[j]$ in $A$ and follows the rule: 
\begin{align*}
rel_{B}[j] = \begin{cases} rel_{A}[i], \ \ i\!f \ (\exists i \ : \  A[i] = B[j]) \\
0, \ \ \ \ \ \ \ \ \ otherwise
\end{cases}
\label{}
% \end{array}
\end{align*} 
Then $B$'s corresponding $rel_{B}[j]$ and $\frac{rel_{B}[j]}{\log_{2}(j+1)}$ values are shown in Table~\ref{estimated_DCG_table}. Accordingly, $B$'s $DCG_{3} = 16+2.524+12.5 = 31.024$. The ranking similarity between $B$ and $A$ is calculated by the ratio of $B$'s DCG to $A$'s DCG, i.e., $\frac{31.024}{39.595} = 0.784$.

\vspace*{2mm}
\noindent
\textbf{Ranking similarity with Euclidean distance
%, $\overline{v D}$,  
input feature.}
% Here we verify if only using Euclidean distance as the input features $f_{s}(v), f_{a}(u) = \langle \overline{v D}, \overline{u D}\rangle$ is sufficient to learn the optimal ranking for the cross-node generalizability and cross-graph generalizability. 
For $\langle f_{s}(v), f_{a}(u) \rangle = \langle \overline{v D}, \overline{u D}\rangle$,
we examine if there exists a linear function $m$ that yields high $SIM_{v}(m, Q^{*})$ and $SIM_{G}(m, Q^{*})$ across different network configurations. Let the ranking metric $m$ be $m(f_{s}(v), f_{a}(u)) = -\overline{u D}$. (Recall that $Q^{*}(f_{s}(v), f_{a}(u))$ is the cumulative negative length of the shortest path.) 

In Figure~\ref{d_node_similarity}, we plot $SIM_{v}(m, Q^{*})$ for all $v, D \in V$ for a given connected uniform random graph  with size in \{50, 100\} and density in \{3, 5\}. Each point represents the value of $SIM_{v}(m, Q^{*})$ for a given $v$ and $D$. All $|V|^2$ points are shown in ascending order. The sub-figures in Figure~\ref{d_node_similarity} illustrate that at least 75\% of the points have high similarity ($\geq 90\%$) between $m$ and $Q^{*}$. According to Theorem~\ref{OPT_Routing}, the distribution of $SIM_{v}(m, Q^{*})$ implies that training samples collected from one (or a few) nodes in this large set can be sufficient to learn a near-optimal routing policy that achieves high accuracy across nodes. On the other hand, using training samples generated from the nodes with relatively low $SIM_{v}(m, Q^{*})$ should be avoided. This observation motivates a knowledge-guide mechanism to carefully choose the data samples used for training.

In Figure~\ref{d_graph_similarity}, we show the distribution of $SIM_{G}(m, Q^{*})$ in ascending order for a set of 100 graphs, respectively with size in \{50, 100\} and density in \{3, 5\}. Note that in high density networks (say density 4), all 100 graphs have high similarity ($\geq 90\%$) between $m$ and $Q^{*}$, implying that training with samples from almost any  high density graph can be sufficient to learn a routing policy with high performance. On the other hand, in low density networks (say density 2), there exist a small set of graphs with slightly low $SIM_{G}(m, Q^{*})$, pointing to the importance of careful selection of seed graph(s) for training. In addition, we observe that the upper bound of $SIM_{G}(m, Q^{*})$ decreases as the network size increases. These  results implicitly show that, with respect to the chosen input feature, graphs with smaller size but higher density can be better seeds for learning generalized routing policies.

According to the results in Figures~\ref{d_node_similarity} and~\ref{d_graph_similarity}, using Euclidean distance implies the existence of a ranking metric that should with high probability satisfy cross-node generalizability and cross-graph generalizability. The ranking function, $m(f_{s}(v), f_{a}(u)) = -\overline{u D}$, or an analogue can be easily learned by a DNN, and then the learned routing policy should achieve high performance across unit-disk uniform random graphs.

\begin{figure}[hbt!]
    \centering
    \subfigure[Size 50, Density 3, rnd=19]
    {
        \includegraphics[width=0.22\textwidth]{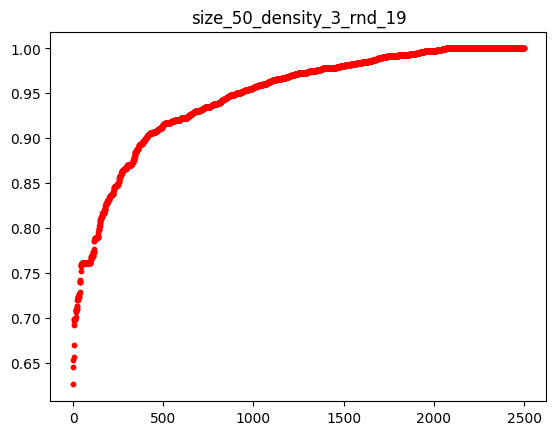}
        \label{d_node_similarity_size50_density3}
    }
    \subfigure[Size 50, Density 5, rnd=19]
    {
        \includegraphics[width=0.22\textwidth]{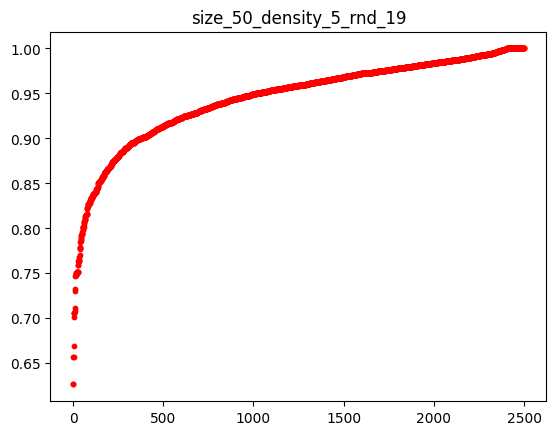}
        \label{d_node_similarity_size50_density5}
    }
    \subfigure[Size 100, Density 3, rnd=48]
    {
        \includegraphics[width=0.22\textwidth]{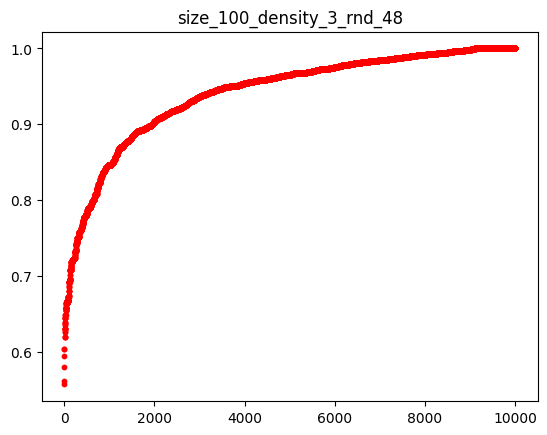}
        \label{d_node_similarity_size100_density3}
    }
    \subfigure[Size 100, Density 5, rnd=48]
    {
        \includegraphics[width=0.22\textwidth]{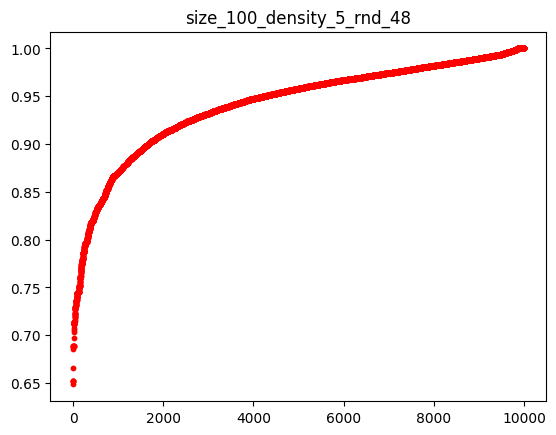}
        \label{d_node_similarity_size100_density5}
    }
    \caption{Distribution of $SIM_{v}(m, Q^{*})$ given a unit-disk uniform random graph with random seed (rnd), where $m$ is the ranking metric for Euclidean distance $\overline{u D}$.}
    \label{d_node_similarity}
%    \vspace{-2mm}
\end{figure}

\begin{figure}[hbt!]
    \centering
    \subfigure[Size 50, Density 3]
    {
        \includegraphics[width=0.22\textwidth]{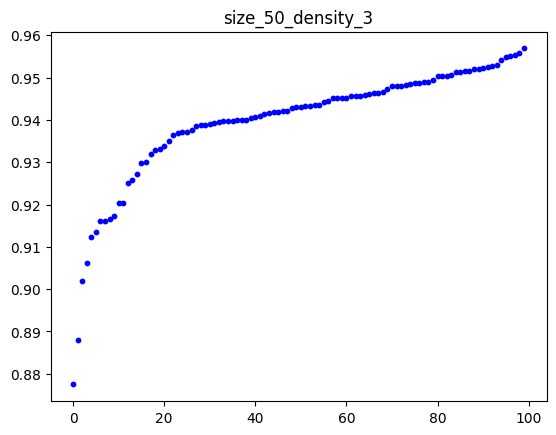}
        \label{d_graph_similarity_size50_density3}
    }
    \subfigure[Size 50, Density 5]
    {
        \includegraphics[width=0.22\textwidth]{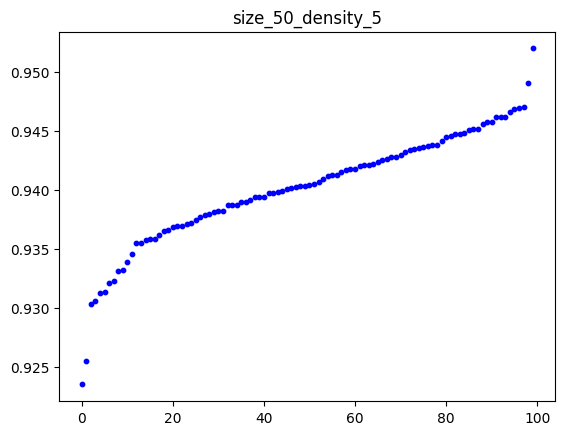}
        \label{d_graph_similarity_size50_density5}
    }
    \subfigure[Size 100, Density 3]
    {
        \includegraphics[width=0.22\textwidth]{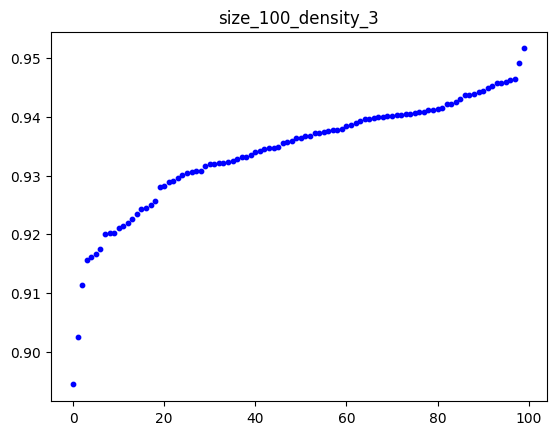}
        \label{d_graph_similarity_size100_density3}
    }
    \subfigure[Size 100, Density 5]
    {
        \includegraphics[width=0.22\textwidth]{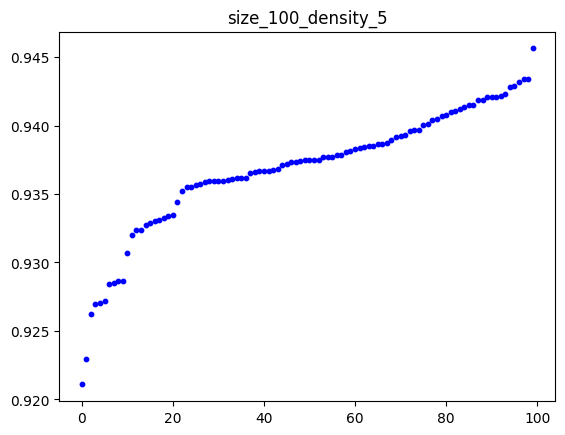}
        \label{d_graph_similarity_size100_density5}
    }
    \caption{Distribution of $SIM_{G}(m, Q^{*})$ across 100 unit-disk uniform random graphs, where $m$ is the ranking metric for Euclidean distance $\overline{u D}$.}
    \label{d_graph_similarity}
    \vspace{-2mm}
\end{figure}

\vspace*{1mm}
\noindent
\textbf{Ranking similarity with Euclidean distance and stretch factor
%, $\overline{v D}$ and $S\!F_{O, D, v}$,  
input features.}  We examine if there exists a linear function $m$ that yields high $SIM_{v}(m, Q^{*})$ and $SIM_{G}(m, Q^{*})$ across different network configurations, for $\langle f_{s}(v), f_{a}(u) \rangle = \langle \overline{v D}, \frac{\overline{O v}+\overline{v D}}{\overline{O D}}, \overline{u D}, \frac{\overline{O u}+\overline{u D}}{\overline{O D}}\rangle$.  Let the ranking metric $m$ as $m(f_{s}(v), f_{a}(u)) = -0.875 \overline{u D} - 0.277 \frac{\overline{O u}+\overline{u D}}{\overline{O D}}$. Note that $f_{s}(v)$ does not affect the sorting of $m(f_{s}(v), f_{a}(u))$. For convenience, we assign zero weight for $\overline{v D}$ and $\frac{\overline{O v}+\overline{v D}}{\overline{O D}}$ in $m$.

\vspace*{1mm} In Figures~\ref{d_sf_node_similarity} and~\ref{d_sf_graph_similarity}, we apply the same network configurations as in Figures~\ref{d_node_similarity} and~\ref{d_graph_similarity} to plot the distribution of $SIM_{v}(m, Q^{*})$ and $SIM_{G}(m, Q^{*})$  for the proposed choice of $m$. Note that because the stretch factor relies on $v, O, D \in V$, there are $|V|^3$ points plotted in Figure~\ref{d_sf_node_similarity} with ascending order. The sub-figures in Figure~\ref{d_sf_node_similarity} demonstrate that at least 80\% of the points have similarity above 90\%, which is higher than the corresponding percentage of points in Figure~\ref{d_node_similarity}. These observations indicate that using both Euclidean distance and stretch factor assures the existence of ranking metric, and in turn implies learnability of a DNN that with high probability achieves even better cross-node generalizability, compared to the one using only Euclidean distance in the input features.

Figure~\ref{d_sf_graph_similarity} shows a similar distribution of $SIM_{G}(m, Q^{*})$ to Figure~\ref{d_graph_similarity}. In order to achieve cross-graph generalizability, we need to carefully select a seed graph for training, especially from graphs with moderate size and high density. Also, instead of using all the nodes to generate data samples, using a subset of nodes that avoid relatively low $SIM_{v}(m, Q^{*})$ further improves the cross-node generalizability.

\begin{figure}[hbt!]
    \centering
%    \vspace{-1mm}
    \subfigure[Size 50, Density 3, rnd=19]
    {
        \includegraphics[width=0.22\textwidth]{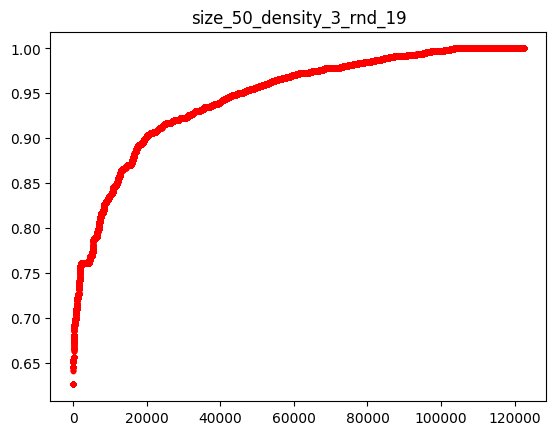}
        \label{d_sf_node_similarity_size50_density3}
    }
    \subfigure[Size 50, Density 5, rnd=19]
    {
        \includegraphics[width=0.22\textwidth]{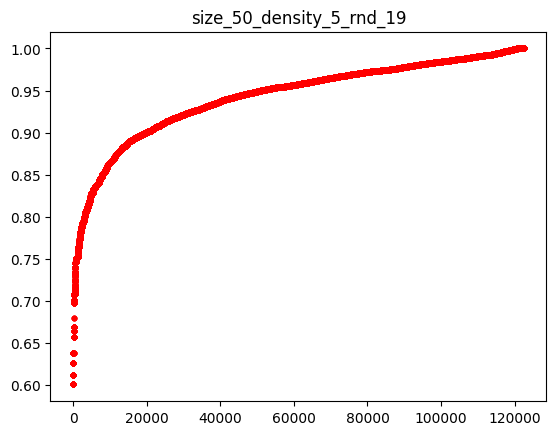}
        \label{d_sf_node_similarity_size50_density5}
    }
    \subfigure[Size 100, Density 3, rnd=48]
    {
        \includegraphics[width=0.22\textwidth]{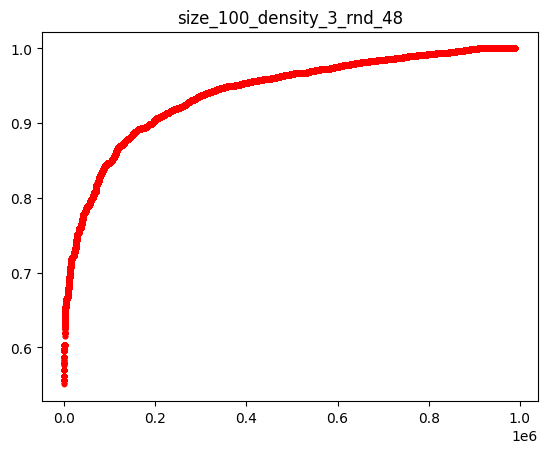}
        \label{d_sf_node_similarity_size100_density3}
    }
    \subfigure[Size 100, Density 5, rnd=48]
    {
        \includegraphics[width=0.22\textwidth]{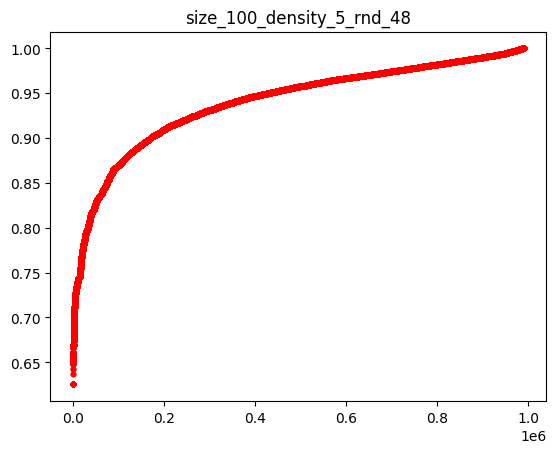}
        \label{d_sf_node_similarity_size100_density5}
    }
    \caption{Distribution of $SIM_{v}(m, Q^{*})$ given a unit-disk uniform random graph with a random seed (rnd), where $m$ is the ranking metric for Euclidean distance $\overline{u D}$ and stretch factor $\frac{\overline{O u}+\overline{u D}}{\overline{O D}}$.}
    \label{d_sf_node_similarity}
%    \vspace{-2mm}
\end{figure}

\begin{figure}[hbt!]
    \centering
%    \vspace{-1mm}
    \subfigure[Size 50, Density 3]
    {
        \includegraphics[width=0.22\textwidth]{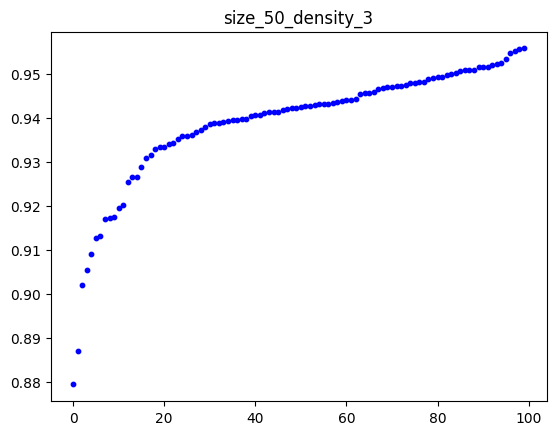}
        \label{d_sf_graph_similarity_size50_density3}
    }
    \subfigure[Size 50, Density 5]
    {
        \includegraphics[width=0.22\textwidth]{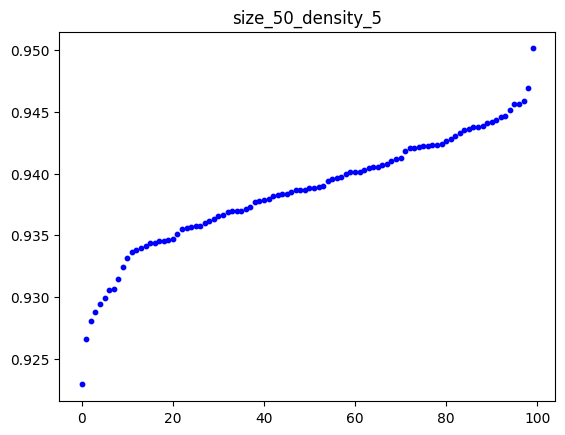}
        \label{d_sf_graph_similarity_size50_density5}
    }
    \subfigure[Size 100, Density 3]
    {
        \includegraphics[width=0.22\textwidth]{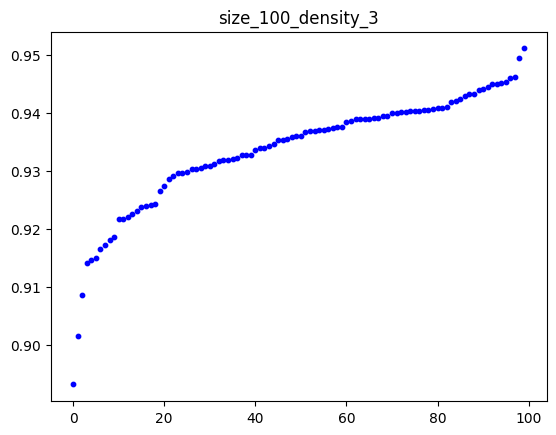}
        \label{d_sf_graph_similarity_size100_density3}
    }
    \subfigure[Size 100, Density 5]
    {
        \includegraphics[width=0.22\textwidth]{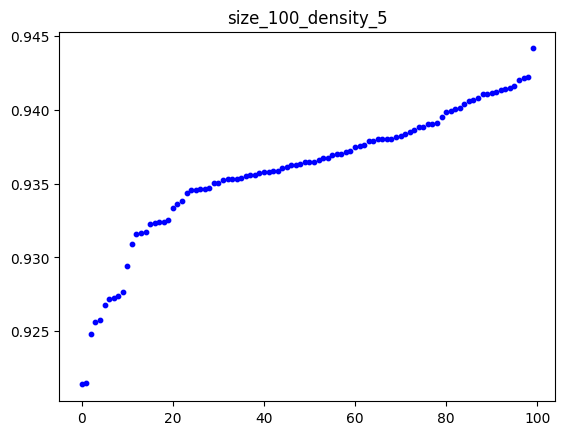}
        \label{d_sf_graph_similarity_size100_density5}
    }
    \caption{Distribution of $SIM_{G}(m, Q^{*})$ across 100 unit-disk uniform random graphs, where $m$ is the ranking metric for Euclidean distance $\overline{u D}$ and stretch factor $\frac{\overline{O u}+\overline{u D}}{\overline{O D}}$.}
    \label{d_sf_graph_similarity}
    % \vspace{-2mm}
\end{figure}

\vspace*{-3mm}
\subsection{Selection of Seed Graph and Graph Subsamples}

Based on our observations in Section V.B, to achieve both cross-graph generalizability and cross-node generalizability, we develop a knowledge-guided mechanism with two components: (1) seed graph selection and (2) graph subsample selection.

\vspace*{1.5mm}
\noindent
\textbf{Seed graph selection.} The choice of seed graph depends primarily on the analysis of $SIM_{G}(m, Q^{*})$ across a sufficient set of uniform random graphs with diverse sizes and densities. 
%Empirically, with the use of Euclidean distance and stretch factor, a good seed graph with $SIM_{G}(m, Q^{*}) \approx 1$ is likely to exist in a set of graphs with small size $\in [10, 50] $ and high density $\in [4, 5]$. 

There may be applications where analysis of large (or full) graphs is not always possible. In such situations, given a graph $G$, an alternative choice of seed graph can be from a small subgraph $G'=(V', E'), V' \subset V, E' \subset E$ with relatively high $SIM_{G}(m, Q^{*})$. Note that, in Theorem~\ref{OPT_Routing}, for a graph $G=(v,E)$ satisfying the {\em RankPres} property for all $v \in V$, {\em RankPres} still holds for $v' \in V'$ in a subgraph $G'=(V', E')$ of $G$. This is because the learnable function $m$ still preserves the optimal routing policy for all nodes in a subset of $nbr(v)$.
%, $u' \in nbr(v') \subset nbr(v)$.

\vspace*{1.5mm}
\noindent
\textbf{Graph subsamples selection.} As observed for $SIM_{v}(m, Q^{*})$ in Section V.B, there may exist a set of nodes, $V_{M}$, with relatively low $SIM_{v}(m, Q^{*})$ in a graph. Using data samples derived from nodes in that set can be harmful to the generalization performance. However, training with samples generated from all nodes in $V \setminus V_{M}$ is not necessary to achieve generalization and may also incur high computational complexity for training. To efficiently choose a set of say $k$ nodes for generating $kd$ training samples, where $d$ denotes the average node degree, we provide the following heuristic.

We assume that, with high probability, we are able to find a random graph $G$ with high $SIM_{G}(m, Q^{*})$ to serve as the seed graph. To limit the search in $G$ for sampling, we search for (one or more) shortest paths $p$ associated with a given destination $D$ each of whose nodes $v$ have a high $SIM_{v}(m, Q^{*})$
%To limit the search for $k$ nodes in $G$ for sampling, we search for (one or more) shortest paths $p$ associated with a given $O$ and $D$ each of whose nodes $v$ have a high $SIM_{v}(m, Q^{*})$.
%Note that, in supervised learning, the candidate path between $O$ and $D$ can be the shortest path to reduce $k$. In RL, the target path between $O$ and $D$ can be the predicted shortest path by the current model. The predicted path is likely to converge to the actual shortest path as the number of completed training episodes increases, so that the number of explored nodes is also limited. 
%To show the efficacy of using the heuristic of subsampling for generalizability, we define $SIM_{p}(m, Q^{*})$ to be the average ranking similarity between $m$ and $Q^{*}$ across all nodes on a given path $p$, i.e., $SIM_{p}(m, Q^{*}) = \frac{\sum_{v \in p}{SIM_{v}(m, Q^{*})}}{|p|}$. 

Let $SIM_{p}(m, Q^{*}) \! = \! \frac{\sum_{v \in p}{SIM_{v}(m, Q^{*})}}{|p|}$ denote the average $SIM_{v}(m, Q^{*})$ for the nodes on a given path $p$. 
%Our heuristic is based on the following results that choosing path(s) with high $SIM_{p}(m, Q^{*})$ leads to a set of chosen nodes with relatively high $SIM_{v}(m, Q^{*})$ compared to $SIM_{G}(m, Q^{*})$. 
In Figures~\ref{d_path_similarity} and~\ref{d_sf_path_similarity_100}, we show the distribution of $SIM_{p}(m, Q^{*})$ for a given graph $G$ with a fixed $D$; the x-axis corresponds to origins $O$ that are sorted in a decreasing order of path stretch $\frac{d_{sp}(O, D)}{d_{e}(O, D)}$: the rightmost point on the x-axis is thus the origin $O$ with the smallest path stretch. 
%Note that there are $|V|-1$ shortest paths from different origins to $D$. 
The results in Figures~\ref{d_path_similarity} and~\ref{d_sf_path_similarity_100} illustrate that for both input feature selections {\em a shortest path with relatively low path stretch is likely to have $SIM_{p}(m, Q^{*})$ close to 1}. 
%This trend holds for both the $m$ for Euclidean distance and the $m$ for Euclidean distance and stretch factor.

Therefore, the policy of subsample selection for a given graph $G=(V,E)$ is specified as follows: 

\begin{enumerate}

\item Select a random destination $D \in V$.

\item Select $\phi$ origins with the lowest path stretch, $v_{0}, ..., v_{\phi-1} \in V \setminus {D}$.

\item For each shortest path $p_{0}, ..., p_{\phi-1}$ associated with the origin-destination pair $(v_{0}, D), ..., (v_{\phi-1}, D)$, respectively, collect the subsamples $\langle X, Y \rangle$, where $X \! = \! \bigcup_{v \in \{p_{0}, ..., p_{\phi-1}\} \wedge u \in nbr(v)}{\{\langle f_{s}(v),f_{a}(u)\rangle \}}$ and $Y \! = \!  \bigcup_{v \in \{p_{0}, ..., p_{\phi-1}\} \wedge u \in nbr(v)}{\{Q^{*}(v, u)\}}$.

\end{enumerate}

% \vspace*{1.5mm}
% Note that the subsampling policy above applies to DNN with supervised learning since the computation of path stretch relies on the knowledge of $Q^{*}$. The following section will adapt the sample selection policy to RL.

% \vspace*{-2mm}
\begin{figure}[thb!]
    \centering
    \subfigure[Size 50, Density 3, $D$=49]
    {
        \includegraphics[width=0.22\textwidth]{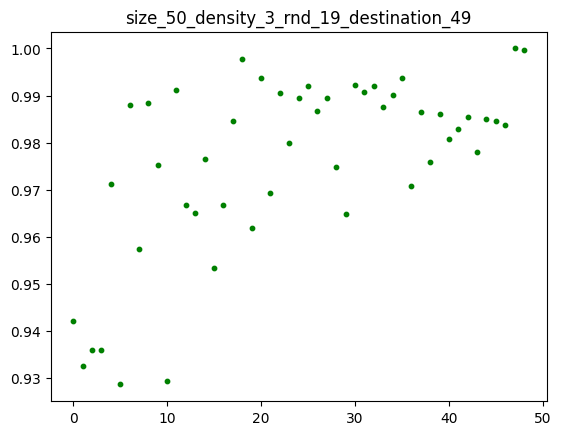}
        \label{d_path_similarity_size50_density3}
    }
    \subfigure[Size 50, Density 5, $D$=49]
    {
        \includegraphics[width=0.22\textwidth]{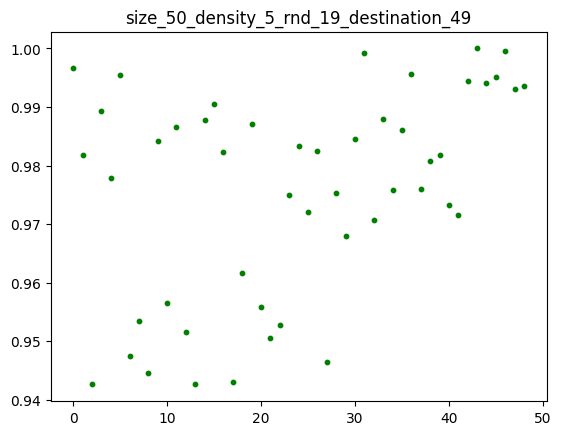}
        \label{d_path_similarity_size50_density5}
    }
    \subfigure[Size 100, Density 3, $D$=10]
    {
        \includegraphics[width=0.22\textwidth]{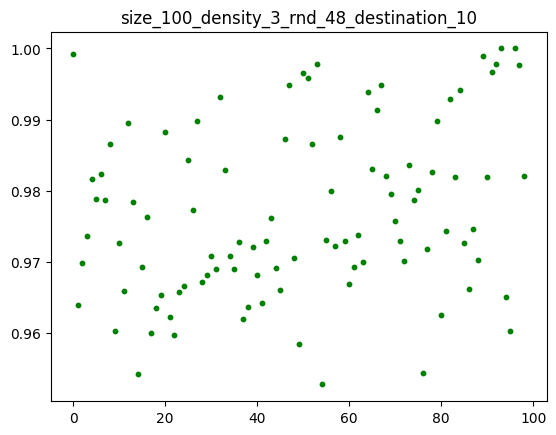}
        \label{d_path_similarity_size100_density3}
    }
    \subfigure[Size 100, Density 5, $D$=10]
    {
        \includegraphics[width=0.22\textwidth]{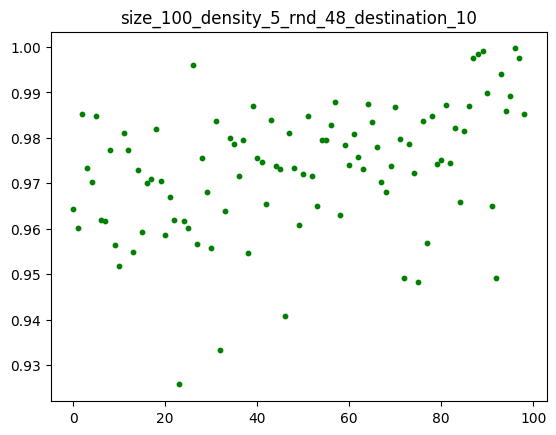}
        \label{d_path_similarity_size100_density5}
    }
    \caption{Distribution of $SIM_{p}(m, Q^{*})$ for all-origins-to-one-destination shortest paths given a unit-disk uniform random graph with random seed (rnd), where $m$ is the ranking metric for Euclidean distance $\overline{u D}$.}
    \label{d_path_similarity}
%    \vspace{-2mm}
\end{figure}

% \vspace*{-2mm}
\begin{figure}[hbt!]
\vspace{2mm}
    \centering
    \subfigure[Size 50, Density 3, $D$=49]
    {
        \includegraphics[width=0.22\textwidth]{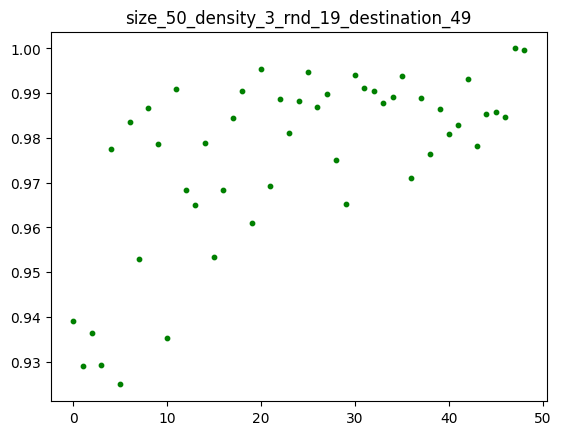}
        \label{d_sf_path_similarity_size50_density3}
    }
    \subfigure[Size 50, Density 5, $D$=49]
    {
        \includegraphics[width=0.22\textwidth]{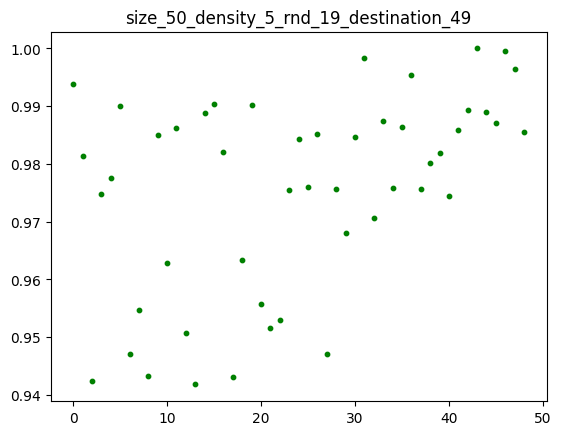}
        \label{d_sf_path_similarity_size50_density5}
    }
    \subfigure[Size 100, Density 3, $D$=10]
    {
        \includegraphics[width=0.22\textwidth]{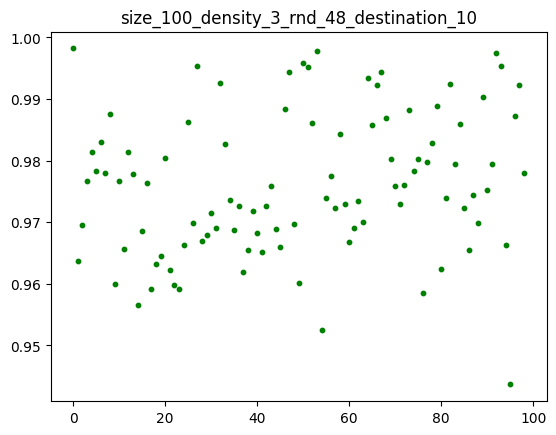}
        \label{d_sf_path_similarity_size100_density3}
    }
    \subfigure[Size 100, Density 5, $D$=10]
    {
        \includegraphics[width=0.22\textwidth]{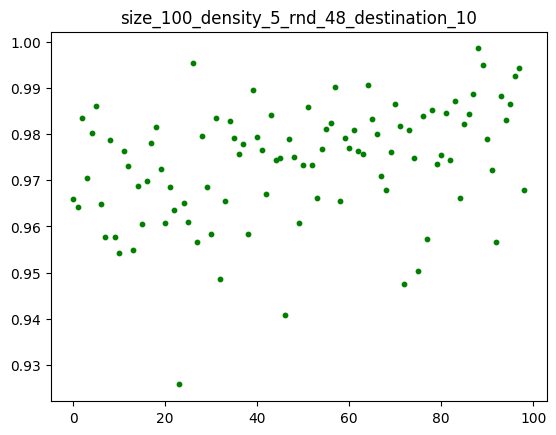}
        \label{d_sf_path_similarity_size100_density5}
    }
    \caption{Distribution of $SIM_{p}(m, Q^{*})$ for all-origins-to-one-destination shortest paths given a unit-disk uniform random graph with random seed (rnd), where $m$ is the ranking metric for $\overline{u D}$ and stretch factor $\frac{\overline{O u}+\overline{u D}}{\overline{O D}}$.}
    \label{d_sf_path_similarity_100}
    % \vspace{-3mm}
\end{figure}

%\vspace*{-1mm}
\subsection{Supervised Learning for APNSP with Optimal $Q$-values}

Given the dataset $\langle X, Y\rangle$ collected by using the subsampling policy for the seed graph, we train a DNN based on supervised learning to capture the optimal ranking policy. Specifically, suppose the DNN $H$ is parameterized by $\theta$. We seek to minimize the following loss function:
\begin{align}\label{supervised}
    \min_{\theta}~~\sum_{\langle X, Y\rangle} \|H_{\theta}(f_{s}(v),f_{a}(u)) - Q^*(v,u)\|^2.
%\vspace*{-1mm}
\end{align}
Note that we assume that the optimal $Q$-values are known for the seed graph in the supervised learning above, which can be obtained based on the shortest path routing policies of the seed graph. By leveraging these optimal $Q$-values and supervised learning on the seed graph, a generalized routing policy is learned for APNSP routing over almost all uniform random graphs, as we validate later in the experiments.

\subsection{Reinforcement Learning for APNSP}
For the case where the optimal $Q$-values of graphs are unknown, we  develop an approach for solving the APNSP problem based on RL. Using the same input features and the same seed graph selection procedure, the RL algorithm continuously improves the quality of $Q$-value estimations by interacting with the seed graph.  

We consider the same DNN architecture as shown in Figure~\ref{DNN}. In contrast to the supervised learning algorithm, where we collect only a single copy of data samples from a set of chosen (shortest path) nodes once before training, in RL new training data samples from nodes in a shortest path, predicted by most recent training episode (i.e., based on the current $Q$-value estimation), are collected at the beginning of each training episode. Remarkably, the generalizability of the resulting RL routing policy across almost all graphs in $U$ is preserved.
%, if the {\em RankPres} property is preserved with respect to a ranking metric $m(f_{s}(v), f_{a}(u))$ of the given input features $f_{s}(v), f_{a}(u)$. 

The details of the RL algorithm, named as RL-APNSP-ALGO, are shown in Algorithm~\ref{RL_ALGO}. More specifically, in Algorithm~\ref{RL_ALGO}, Lines 2 to 20 outline the sample selection and training procedure for each episode. The for-loop from Lines 4 to 7 determines the set of chosen shortest paths and the associated nodes for subsampling, where the shortest paths are predicted based on the current DNN estimation of $Q(v, u)$ for a routing node $v$ and its neighbors $u \in nbr(v)$. The neighbor $u$ is chosen to be the next routing node until the destination is reached or none of the neighbors remain unvisited. Note that each node can be only visited once in a round of path exploration. It is possible to output a path $p$ without the destination $D$, which  can still be used for subsampling in this episode. The for-loop in Lines 10 to 17 generates the training data samples from nodes shown in the chosen paths, where the data labels $Y$, i.e., target $Q$-values that we train the DNN to fit for improving the estimation accuracy of optimal $Q$-values, are given by: 
\begin{align}
    Q^{target}(s, a) = r(s, a) \! + \! \gamma \max_{a'}{Q(s', a')}.
\label{Q_estimate}
%\vspace*{-2mm}
\end{align}
Next, based on the collected dataset, in Lines 18 to 20, the DNN is trained for a fixed number of iterations to minimize the following loss function:
\begin{align}\label{fit}
    \min_{\theta}~~\sum_{\langle X, Y\rangle} \|H_{\theta}(f_{s}(v),f_{a}(u)) - Q^{target}(v,u)\|^2.
%\vspace*{-1mm}
\end{align}
% The estimation of $Q$-value in RL relies on the prediction of $Q$-value by the currently learned model and the observation on the reward for $(s, a)$. According to Q-learning, the estimate of $Q$-value is defined as follows:
% \begin{align}
%     Q^{new}(s_t, a_t) =  Q(s_t, a_t) \! + \! \alpha  [r(s, a) \! + \! \gamma \max_{a}{Q(s_{t+1}, a')}]
% \label{Q_estimate}
% \vspace*{-2mm}
% \end{align}
% where $\alpha, 0 < \alpha \leq 1$ is a constant  learning rate. 
Since the target $Q$-values approach the optimal $Q$-values as the number of training episodes increases, minimizing Equation~\ref{fit} will eventually lead to a learned model that nearly matches the supervised learning in Equation~\ref{supervised}.

\begin{algorithm}
\SetAlgoLined
Input: $nn$: randomly initialized DNN; $G^{*}$: seed graph; $\Phi$: set of chosen sources; $D$: chosen destination

\For{$episode = 1 ... EpiNum$}{
$V_{T}$ := \{\}\;
\For{$O \in \Phi$}{
  Use $nn$ to predict a shortest path $p$ for $(O, D)$ in $G^{*}$\;
  $V_{T}$ := $V_{T} \cup \{v| v \in p\}$\;
}
$X$:=[], $Y$:=[], 
$i := 0$\;
\For{$v \in V_{T}$}{
  \For{$u \in nbr(v)$}{
    $X[i] := \langle f_{s}(v),f_{a}(u)\rangle$\;
    Estimate $Q(v, u)$ using Equation~\ref{Q_estimate}\;
    $Y[i] := Q(v, u)$\;
    $i := i+1$\;
    }
}
\For{$iter = 1 ... IterNum$}{
  Train $nn$ with $\langle X, Y \rangle$ based on Equation~\ref{fit}\;
}
}

\Return $nn$\;
\caption{RL-APNSP-ALGO}
\label{RL_ALGO}
\end{algorithm}

% \section{Improvement to Learned Model}
% \subsection{Fine-tuning for target graphs}

\section{Routing Policy Performance for Scalability and Zero-shot Generalization}

In this section, we discuss implementation of our machine learned routing policies and evaluate their performance in predicting all-pair near-shortest paths for graphs across different sizes and densities in Python3. We use PyTorch 2.0.1 \cite{pytorch} on the CUDA 11.8 compute platform to implement DNNs as shown in Figure~\ref{DNN}. Table~\ref{SimulationParams} shows our simulation parameters for training and testing the routing policies. 

\vspace*{-1mm}
\begin{table}[htbp]
\caption{Simulation Parameters}
\begin{center}
\begin{tabular}{c c c}
\hline
\textbf{Symbol}& \textbf{Meaning} & \textbf{Value} \\
\hline

$N_{train}$ & size of seed graph & 50 \\
$\rho_{train}$ & density of seed graph & 5 \\
$N_{test}$ & sizes of tested graphs & 27, 64, 125 \\
$\rho_{test}$ & densities of tested graphs & 2, 3, 4, 5 \\
$R$ & communication radius & 1000 \\
$\Omega=I+J$ & \# of input features & 2,4 \\
$K$ & \# of hidden layers & 2 \\
$N_{e}[]$ & \# of neurons in each hidden layer  & $[50\Omega, \Omega]$ \\
$\epsilon$ & margin for shortest paths prediction & 0.05 \\
$\phi$ & \# of origins for subsampling & 3 \\
$\gamma$ & discount factor & 1 \\
$IterNum_{S}$ & \# of iterations in supervised learning & 5000 \\ 
$IterNum_{RL}$ & \# of iterations in RL & 1000 \\ 
$EpiNum$ & \# of episodes in RL & 20 \\ 
\hline
\end{tabular}
\label{SimulationParams}
\end{center}
\vspace{-2mm}
\end{table}

%\vspace*{-1mm}
\subsection{Comparative Evaluation of Routing Policies}
We compare the performance of the routing policies obtained using the following approaches:
\begin{itemize}
\item {\bf Supervised ($\phi$=3)}:  Supervised learning  from appropriately chosen seed graph $G^{*}$ using graph subsamples selection with $\phi$=3 (see Section V.C).
\item {\bf Supervised (all)}: Supervised learning  from appropriately chosen seed graph $G^{*}$ with samples generated from all nodes.
\item {\bf RL ($\phi$=3)}: RL from appropriately chosen seed graph $G^{*}$ using graph subsamples selection with $\phi$=3.
\item {\bf RL (all)}: RL from appropriately chosen seed graph $G^{*}$ with samples generated from all nodes.
\item {\bf GF}: Greedy forwarding that forwards packets to the one-hop neighbor with the minimum Euclidean distance to the destination. 
\end{itemize}
Note that for a given set of input features, both supervised learning and reinforcement learning schemes use the same DNN configuration to learn the routing policies. 
By using the subsampling mechanism, not only the sample complexity but also the training time will be significantly reduced in  Supervised ($\phi$=3) and  RL ($\phi$=3) compared to those in Supervised (all) and  RL (all), respectively.

% \begin{figure*}[hbt!]
%     \centering
%     \subfigure[Size 27]
%     {
%         \includegraphics[width=0.315\textwidth]{simulation/testing_performance_N27.png}
%         \label{testing_size27}
%     }
%     \subfigure[Size 64]
%     {
%         \includegraphics[width=0.315\textwidth]{simulation/testing_performance_N64.png}
%         \label{testing_size64}
%     }
%     \subfigure[Size 125]
%     {
%         \includegraphics[width=0.315\textwidth]{simulation/testing_performance_N125.png}
%         \label{testing_size125}
%     }
%     \caption{APNSP prediction accuracy (\%) on tested graphs of various sizes and densities for the routing policies based on both Euclidean distance and stretch factor}
%     \label{Testing_Performance_ed_sf}
% \end{figure*}

\begin{figure*}[hbt!]
    \centering
    \vspace{-2mm}
    \subfigure[Density 2]
    {
        \includegraphics[width=0.23\textwidth]{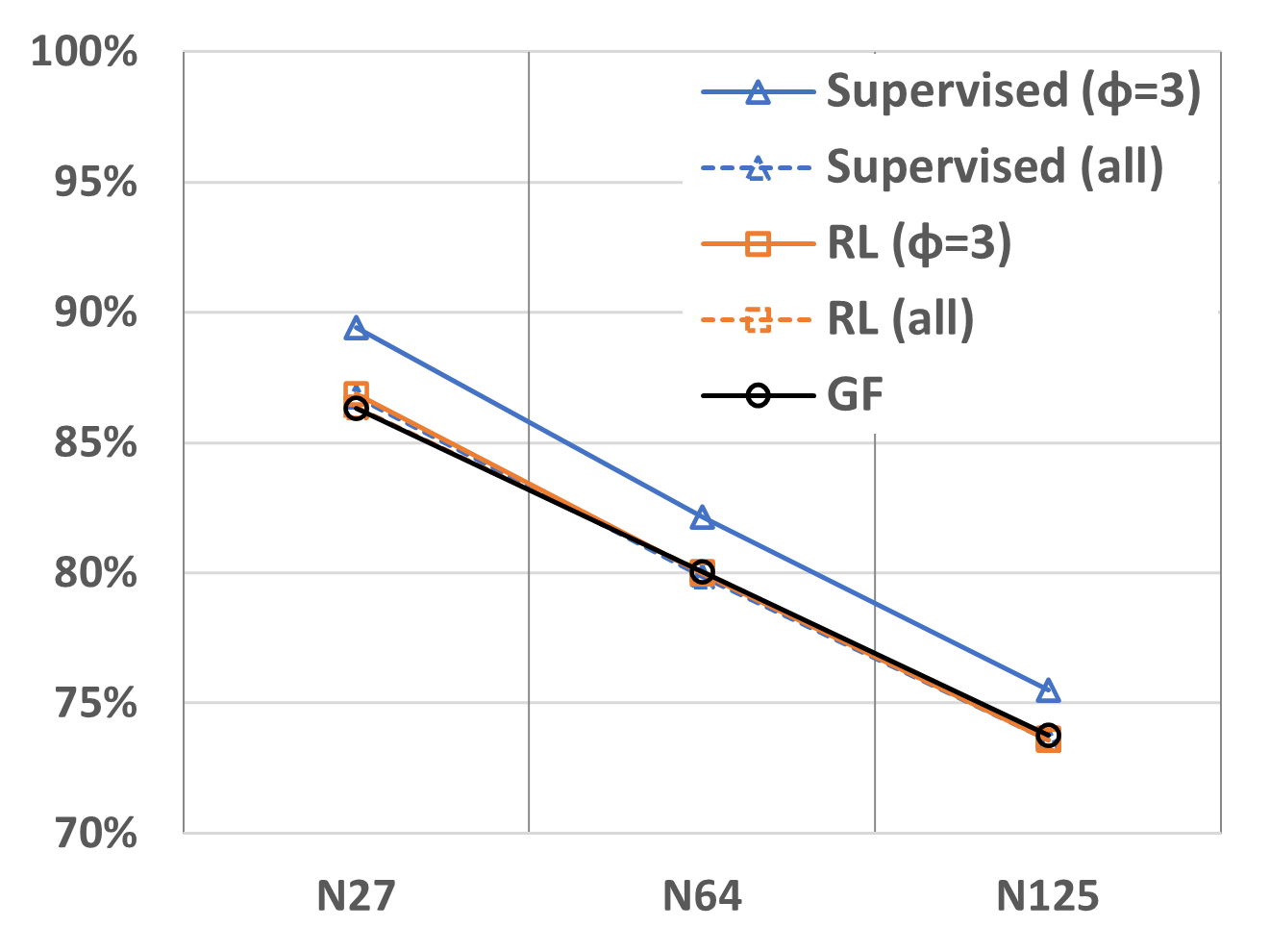}
        \label{testing_density2}
    }
    \subfigure[Density 3]
    {
        \includegraphics[width=0.23\textwidth]{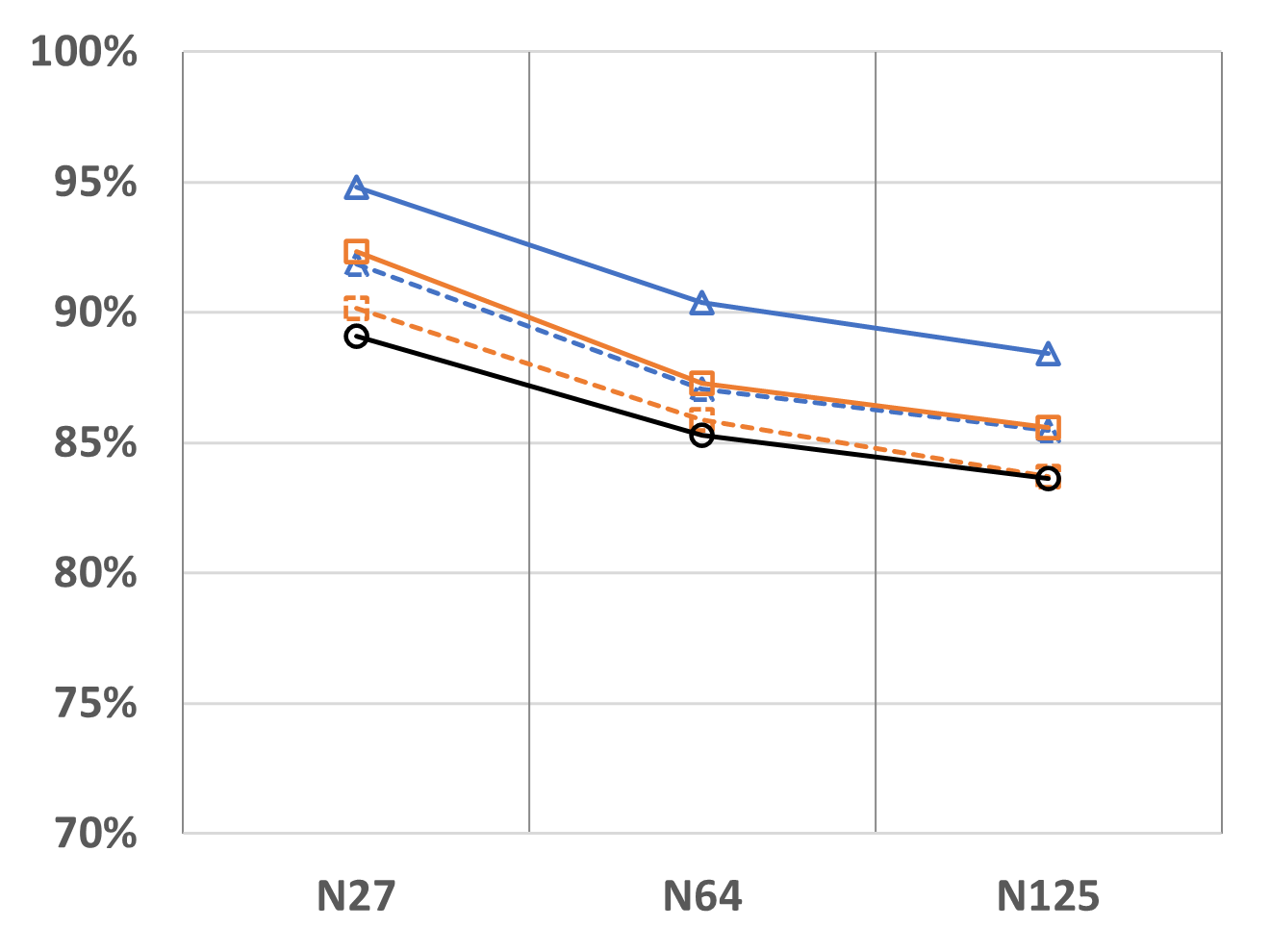}
        \label{testing_density3}
    }
    \subfigure[Density 4]
    {
        \includegraphics[width=0.23\textwidth]{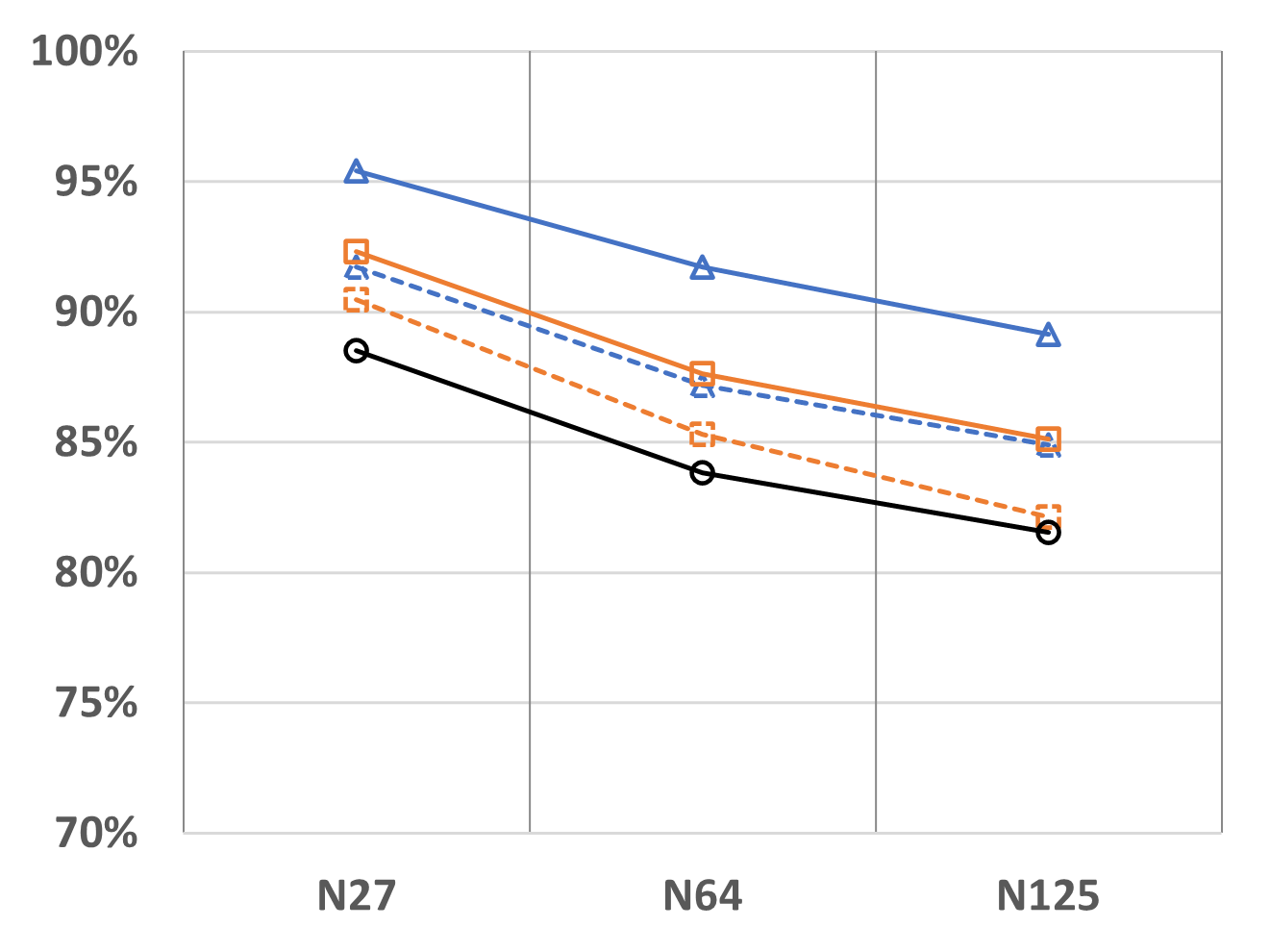}
        \label{testing_density4}
    }
    \subfigure[Density 5]
    {
        \includegraphics[width=0.23\textwidth]{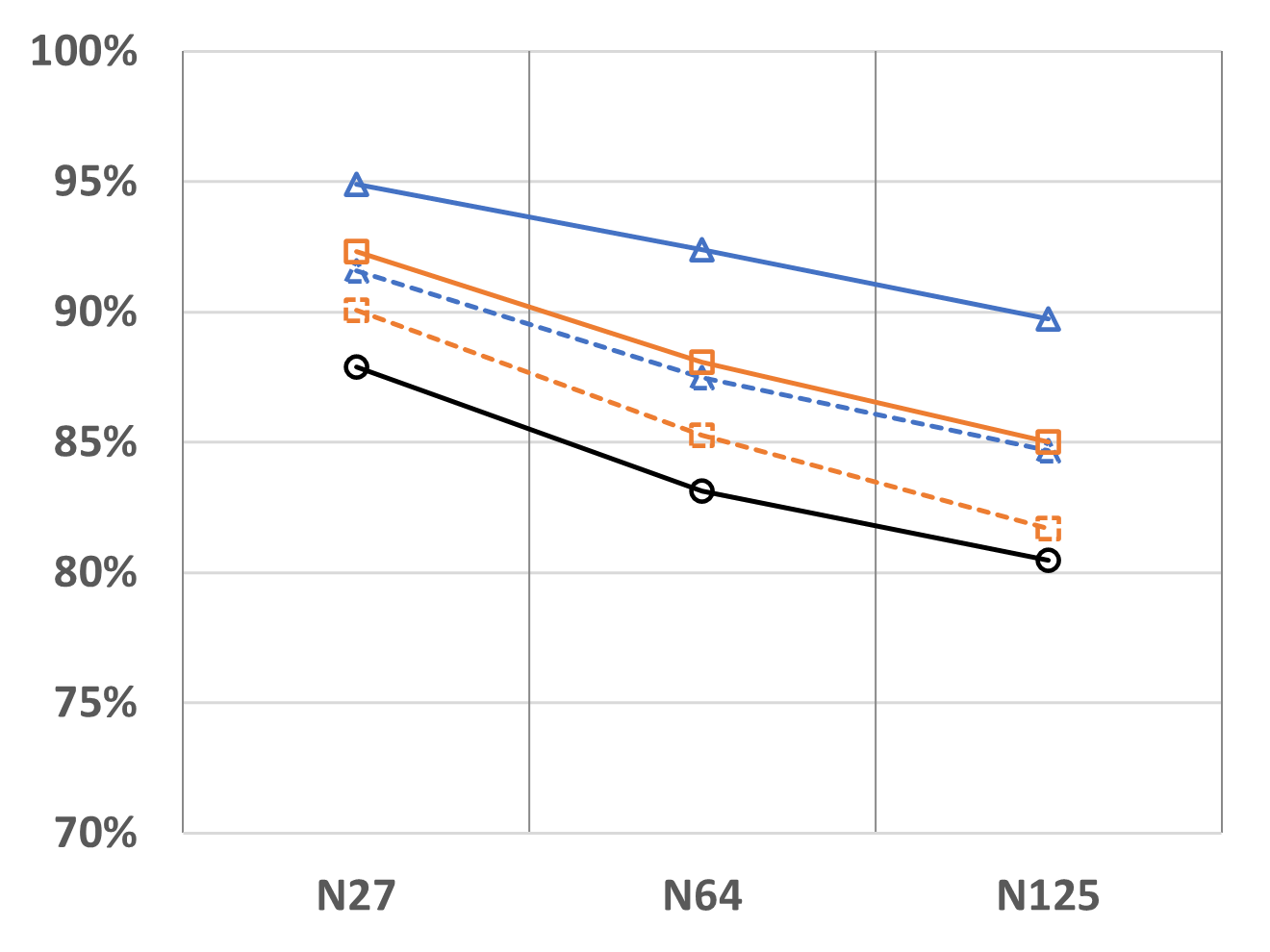}
        \label{testing_density5}
    }
    \caption{APNSP prediction accuracy on tested graphs of various sizes and densities for the routing policies based on both Euclidean distance and stretch factor.}
    \label{Testing_Performance_ed_sf}
    \vspace{-4mm}
\end{figure*}

\subsection{Testing Performance of Routing Policies for $G^*$}
Using the seed graph selection procedure discussed in Section V.C, we choose a seed graph $G^*$, with size 50, density 5, and $SIM_{G}(m, Q^{*})=0.943$.
%, from a set of uniform random graphs $G$ 
%of modest size and high density 
%that have high $SIM_{G}(m, Q^{*})$. 
The testing performance of the learned routing policies on $G^{*}$ with the DNNs under the two input feature sets is shown respectively in Tables~\ref{Training_Performance_ed} and \ref{Training_Performance_ed_sf}, where the prediction accuracy is computed based on Equation~\ref{APNSP_Accuracy} for all pairs of $(O, D), O, D \in V$.
\vspace*{-2mm}
\begin{table}[htbp]
\caption{APNSP prediction accuracy (\%) on G* for the routing policies considering only Euclidean distance}
\begin{center}
\begin{tabular}{|c|c|c|c|c|}
\hline
Supervised($\phi$=3) & Supervised(all) & RL($\phi$=3) & RL(all) & GF \\
\hline
84.00\% & 84.00\% & 84.00\% & 84.00\% & 84.00\% \\
\hline
\end{tabular}
\label{Training_Performance_ed}
\end{center}
\vspace{-2mm}
\end{table}

%\noindent
%\textbf{Using Euclidean distance as the input feature in DNNs.} 
Table~\ref{Training_Performance_ed} shows that the DNNs trained using our proposed machine learning approaches with the input features of $\overline{v D}$ and $\overline{u D}$ exactly match the performance of GF on $G^{*}$. The result implies that greedy forwarding 
%(or a linear function $m(f_{s}(v), f_{a}(u)) = -\overline{u D}$) 
is learnable by DNNs with Euclidean distance based input features. Individual experiments over  hundreds of graphs with various sizes and densities show that both Supervised ($\phi$=3) and RL ($\phi$=3) learn a routing policy whose performance matches that of GF for all of those graphs. We also observe that as the graph size increases or the density decreases, training the DNNs with the samples from all nodes tends to yield a routing policy whose accuracy is lower than that of greedy forwarding, presumably because including samples from nodes whose $SIM_{v}(m, Q^{*})$ is low hurts generalizability.
\vspace*{-2mm}
\begin{table}[htbp]
\caption{APNSP prediction accuracy (\%) on G* for the routing policies considering both Euclidean distance and stretch factor}
\begin{center}
\begin{tabular}{|c|c|c|c|c|}
\hline
Supervised($\phi$=3) & Supervised(all) & RL($\phi$=3) & RL(all) & GF \\
\hline
93.18\% & 88.12\% & 89.14\% & 87.14\% & 84.00\% \\
\hline
\end{tabular}
\label{Training_Performance_ed_sf}
\end{center}
\vspace{-2mm}
\end{table}
%\noindent
%\textbf{Using Euclidean distance and stretch factor as the input feature in DNNs.} 

For the same seed graph $G^{*}$, we train the DNNs with the input features of both Euclidean distance and stretch factor using our proposed approaches, and show that their testing performance on $G^*$ is better than that of GF as demonstrated in Table~\ref{Training_Performance_ed_sf}. Notably, the {\em DNNs trained with samples from the chosen nodes achieve higher prediction accuracy compared to those trained with samples from all nodes}. In the comparison between the supervised learning and the RL schemes, Supervised ($\phi$=3) (respectively, Supervised (all)) shows better performance than RL ($\phi$=3) (respectively, RL (all)), which is expected given the knowledge of optimal $Q$-values in the supervised learning schemes. Note that both RL schemes substantially improve over GF even without knowing the optimal $Q$-values, corroborating the benefits of incorporating domain knowledge in machine learned routing.

Regarding the seed graph selection, we also conduct experiments with a variety of seed graphs for both input feature sets. For seed graphs $G^*$ with high $SIM_{G}(m, Q^{*})$, training with graphs of small size (e.g., of size in 10..50) and high density can achieve higher performance compared to training with large sized graphs in both Superivsed and RL schemes.
 We also observed that competitive performance can be achieved by learning not only from seed graphs with high $SIM_{G}(m, Q^{*})$ but also from seed graphs with modest $SIM_{G}(m, Q^{*})$, given that the training samples are selected from nodes $v$ with high $SIM_{v}(m, Q^{*})$. These results again suggest  avoiding samples that harm generalizability.

%, they achieve high generalization performance by training with a seed graph with a small size (e.g., [10, 50]) and a high density (e.g., $\geq$ 4). However, our heuristic for graph subsamples selection loosens the limit of graph size in seed graph selection. This is because the subsampling scheme only considers nodes in the shortest path with high path stretch and, thus, with high probability, avoiding using samples that harm the generalizability. 

\subsection{Zero-shot Generalization over Diverse Graphs}

% \noindent
% \textbf{Using Euclidean distance as the input feature in DNNs.}

% \noindent
% \textbf{Using Euclidean distance and stretch factor as the input feature in DNNs.}

To evaluate the scalability and generalizability of the routing policies, we directly (i.e., without any adaptation) test the policies learned from the seed graph $G^*$ on new unit-disk uniform random graphs with different combinations of $(N_{test}, \rho_{test})$. We select 20 random graphs for each pair and calculate the average prediction accuracy over these $20{N_{test}}^2$ shortest paths. 

For the DNNs with input  $\overline{v D}$ and $\overline{u D}$, the tests confirm that the performance of all the learned policies match the prediction accuracy of GF.

For the DNNs with input $\langle \overline{v D}, \frac{\overline{O v}+\overline{v D}}{\overline{O D}}, \overline{u D}, \frac{\overline{O u}+\overline{u D}}{\overline{O D}}\rangle$, we plot in Figure~\ref{Testing_Performance_ed_sf} the respective  prediction accuracies across graphs with size in \{27, 64, 125\} and density in \{2, 3, 4, 5\}. The Supervised ($\phi$=3) approach achieves the best performance among all the approaches. In particular, compared to GF, the Supervised ($\phi$=3) policy improves the accuracy up to 10\% over GF, whereas the other learned policies show at least comparable performance in low density graphs ($\rho=2$) and achieve an improvement of up to 6\% in graphs with $\rho \geq 3$. The performance gap between the DNNs and GF increases as the network density increases to a high level (e.g., $\rho=5$), wherein GF was believed to work close to the optimal routing. 

The superior generalization performance of the routing policies learned by Supervised ($\phi$=3) validates that, learning from a single carefully chosen seed graph, with the knowledge of its shortest paths, is sufficient to obtain a generalized routing policy for direct applications in solving the APNSP problem in almost all random graphs.

\vspace*{2mm}
\section{Conclusions and Future Work}
We have shown that guiding machine learning with domain knowledge can lead to the rediscovery of well-known routing algorithms (sometimes surprisingly), in addition to new routing policies that perform well in terms of complexity, scalability, and generalizability. The little theory we have presented in the paper is readily extended to other classes of graphs (such as scale free graphs or non uniform cluster distributions), ranking metrics that are nonlinear, and MDP actions that span multiple neighbors. Thus, albeit our illustration intentionally uses relatively familiar input features and local routing architectures, richer domain theory will be useful to guide machine learning of novel routing algorithms. We also note that, while we have presented empirical results for networks up to size 125, our results have been tested on instances of larger networks of several hundred to a 1000 nodes. Moreover, the routing policies are likely to be competitive for richer classes of graphs than the set of unit-disk uniform random graphs on which we have focused our validation.

While samples from nodes in a single shortest path of a single seed graph suffice for generalizable learning, in practice, learning from multiple shortest paths in one or more seed graphs may be of interest.  For instance, if an ideal seed graph or shortest path is not known a priori, online learning from better or multiple candidate seed graphs and paths as they are encountered may be of interest for some applications. Along these lines, we recall that the set of ideal (and near ideal) seed graphs is relatively large in the problem we considered. One way to relax the knowledge of ideal seed graphs is to leverage online meta-learning, for learning a good model initialization and continuing to improve the initialization based on better seed graphs as they are encountered. Towards this end, we have also been studying the merits of efficiently fine tuning the model for the target graph as an alternative to zero-shot generalization.  

%Comments on learning from "algorithms"
%%%%%%%%%%%%%%%%%%%%%%%%%%%%%%%%%%%%%%%%%%%%%%%%%%%%%%%%%%%%%%%%%%%%%%%%%%%%%%%
\newpage
\bibliographystyle{IEEEtran}
\bibliography{myref}
\end{document}